\documentclass[11pt]{article}

\usepackage{vitercik}

\newcommand{\tree}{\mathcal{T}}
\newcommand{\node}{Q}
\newcommand{\score}{\textnormal{\texttt{score}}}
\newcommand{\merge}{{\xi}}
\newcommand{\gp}{\textsc{Partition}}

\title{Learning to Optimize Computational Resources:\\
Frugal Training with Generalization Guarantees}

\author{
Maria-Florina Balcan \\ \small Carnegie Mellon University \\ \small \texttt{ninamf@cs.cmu.edu}
\and 
Tuomas Sandholm \\ \small Carnegie Mellon University \\
\small Optimized Markets, Inc.\\
\small Strategic Machine, Inc.\\
\small Strategy Robot, Inc.\\
 \small \texttt{sandholm@cs.cmu.edu}
\and 
Ellen Vitercik \\ \small Carnegie Mellon University \\ \small \texttt{vitercik@cs.cmu.edu}}

\begin{document}

\maketitle

\begin{abstract}
	Algorithms typically come with tunable parameters that have a considerable impact on the computational resources they consume. Too often, practitioners must hand-tune the parameters, a tedious and error-prone task. A recent line of research provides algorithms that return nearly-optimal parameters from within a finite set. These algorithms can be used when the parameter space is infinite by providing as input a random sample of parameters. This data-independent discretization, however, might miss pockets of nearly-optimal parameters: prior research has presented scenarios where the only viable parameters lie within an arbitrarily small region. We provide an algorithm that learns a finite set of promising parameters from within an infinite set. Our algorithm can help compile a configuration portfolio, or it can be used to select the input to a configuration algorithm for finite parameter spaces. Our approach applies to any configuration problem that satisfies a simple yet ubiquitous structure: the algorithm's performance is a piecewise constant function of its parameters. Prior research has exhibited this structure in domains from integer programming to clustering.
\end{abstract}

\section{Introduction}
Similar combinatorial problems often arise in seemingly unrelated disciplines. Integer programs, for example, model problems in fields ranging from computational biology to economics. To facilitate customization, algorithms often come with tunable parameters that significantly impact the computational resources they consume, such as runtime. Hand-tuning parameters can be time consuming and may lead to sub-optimal results. In this work, we develop the foundations of automated algorithm configuration via machine learning. A key challenge we face is that in order to evaluate a configuration's requisite computational resources, the learning algorithm itself must expend those resources.

To frame algorithm configuration as a machine learning problem, we assume sample access to an unknown distribution over problem instances, such as the integer programs an airline solves day to day. The learning algorithm uses samples to determine parameters that, ideally, will have strong expected performance.
Researchers have studied
this configuration model for decades, leading to
advances in artificial intelligence~\citep{Xu08:SATzilla}, computational
biology~\citep{DeBlasio18:Parameter}, and many other fields. This approach has also been used in industry for tens of billions of dollars of combinatorial
auctions~\citep{Sandholm13:Very-Large-Scale}.

Recently, two lines of research have emerged that explore the theoretical underpinnings of algorithm configuration. One provides sample complexity guarantees, bounding the number of samples sufficient to ensure that an algorithm's performance on average over the samples generalizes to its expected performance on the distribution~\citep{Gupta17:PAC,Balcan17:Learning,Balcan18:Learning}. These sample complexity bounds apply no matter how the learning algorithm operates, and these papers do not include learning algorithms that extend beyond exhaustive search.

The second line of research provides algorithms for finding nearly-optimal configurations from a finite set~\citep{Kleinberg17:Efficiency,Kleinberg19:Procrastinating,Weisz18:LEAPSANDBOUNDS,Weisz19:CapsAndRuns}. These algorithms can also be used when the parameter space is infinite: for any $\gamma \in (0,1)$, first sample $\tilde{\Omega}(1 / \gamma)$ configurations, and then run the algorithm over this finite set. The authors guarantee that the output configuration will be within the top $\gamma$-quantile. If there is only a small region of high-performing parameters, however, the uniform sample might completely miss all good parameters. Algorithm configuration problems with only tiny pockets of high-performing parameters do indeed exist: \citet{Balcan18:Learning} present distributions over integer programs where the optimal parameters lie within an arbitrarily small region of the parameter space. For any parameter within that region, branch-and-bound---the most widely-used integer programming algorithm---terminates instantaneously. Using any other parameter, branch-and-bound takes an exponential number of steps. This region of optimal parameters can be made so small that any random sampling technique would require an arbitrarily large sample of parameters to hit that region. We discuss this example in more detail in Section~\ref{sec:comparison}.

This paper marries these two lines of research. We present an algorithm that identifies a
finite set of promising parameters within an infinite set, given sample access to a distribution over problem instances. We prove that this set contains a nearly
optimal parameter with high probability. The set can serve as the input to a configuration algorithm
for finite parameter spaces~\citep{Kleinberg17:Efficiency,Kleinberg19:Procrastinating,Weisz18:LEAPSANDBOUNDS,Weisz19:CapsAndRuns}, which we prove will then return a nearly optimal parameter from the
infinite set.

An obstacle in our approach is that the loss function measuring an algorithm's performance as a function of its parameters often exhibits jump discontinuities: a nudge to the parameters can trigger substantial changes in the algorithm's behavior. In order to provide guarantees, we must tease out useful structure in the configuration problems we study.

The structure we identify is simple yet ubiquitous in combinatorial domains:
our approach applies to any configuration problem where the algorithm's performance as a function of its parameters is piecewise constant.
Prior research has demonstrated that algorithm configuration problems from diverse domains exhibit this structure. For example, \citet{Balcan18:Learning} uncovered this structure for branch-and-bound algorithm configuration. Many corporations must regularly solve reams of integer programs, and therefore require highly customized solvers. For example, integer programs are a part of many mesh processing pipelines in computer graphics~\citep{Bommes09:Mixed}. Animation studios with thousands of meshes require carefully tuned solvers which, thus far, domain experts have handcrafted~\citep{Bommes10:Practical}. Our algorithm can be used to find configurations that minimize the branch-and-bound tree size.
\citet{Balcan17:Learning} also exhibit this piecewise-constant structure in the context of linkage-based hierarchical clustering algorithms. The algorithm families they study interpolate between the classic single-, complete-, and average-linkage procedures. Building the cluster hierarchy is expensive: the best-known algorithm's runtime is $\tilde O(n^2)$ given $n$ datapoints~
\citep{Manning10:Introduction}. As with branch-and-bound, our algorithm finds configurations that return satisfactory clusterings while minimizing the hierarchy tree size.

We now describe our algorithm at a high level. Let $\ell$ be a loss function where $\ell(\vec{\rho}, j)$ measures the computational resources (running time, for example) required to solve problem instance $j$ using the algorithm parameterized by the vector $\vec{\rho}$. Let $OPT$ be the smallest expected loss\footnote{As we describe in Section 2, we compete with a slightly more nuanced benchmark than $OPT$, in line with prior research.} $\E_{j \sim \Gamma}[\ell(\vec{\rho}, j)]$ of any parameter $\vec{\rho}$, where $\Gamma$ is an unknown distribution over problem instances. Our algorithm maintains upper confidence bound on $OPT$, initially set to $\infty$. On each round $t$, the algorithm begins by drawing a set $\sample_t$ of sample problem instances. It computes the partition of the parameter space into regions where for each problem instance in $\sample_t$, the loss $\ell$, capped at $2^t$, is a constant function of the parameters. On a given region of this partition, if the average capped loss is sufficiently low, the algorithm chooses an arbitrary parameter from that region and deems it ``good.'' Once the cap $2^t$ has grown sufficiently large compared to the upper confidence bound on $OPT$, the algorithm returns the set of good parameters. We summarize our guarantees informally below.

\begin{theorem}[Informal]
	The following guarantees hold:
	\begin{enumerate}
		\item The set of output parameters contains a nearly-optimal parameter with high probability.
		\item Given accuracy parameters $\epsilon$ and $\delta$, the algorithm terminates after $O\left(\ln\left(\sqrt[4]{1+\epsilon} \cdot OPT/\delta\right)\right)$ rounds.
		\item On the algorithm's final round, let $P$ be the size of the partition the algorithm computes. The number of parameters it outputs is $O\left(P \cdot \ln\left(\sqrt[4]{1+\epsilon} \cdot OPT/\delta\right)\right)$.
		\item The algorithm's sample complexity on each round $t$ is polynomial in $2^t$ (which scales linearly with $OPT$), $\log P$, the parameter space dimension, $\frac{1}{\delta}$, and $\frac{1}{\epsilon}$.
	\end{enumerate}
\end{theorem}	

We prove that our sample complexity can be exponentially better than the best-known uniform convergence bound. Moreover, it can find strong configurations in scenarios where uniformly sampling configurations will fail.

\section{Problem definition}\label{sec:def}
The algorithm configuration model we adopt is a generalization of the model from prior research~\citep{Kleinberg17:Efficiency,Kleinberg19:Procrastinating,Weisz18:LEAPSANDBOUNDS,Weisz19:CapsAndRuns}. There is a set $\Pi$ of problem instances and an unknown distribution $\Gamma$ over $\Pi$. For example, this distribution might represent the integer programs an airline solves day to day. Each algorithm is parameterized by a vector $\vec{\rho} \in \cP \subseteq \R^d$. At a high level, we assume we can set a budget on the computational resources the algorithm consumes, which we quantify using an integer $\tau \in \Z_{\geq 0}$. For example, $\tau$ might measure the maximum running time we allow the algorithm. There is a utility function $u: \cP \times \Pi \times \Z_{\geq 0} \to \{0,1\}$, where $u(\vec{\rho}, j, \tau) = 1$ if and only if the algorithm parameterized by $\vec{\rho}$ returns a solution to the instance $j$ given a budget of $\tau$. We make the natural assumption that the algorithm is more likely to find a solution the higher its budget: $u(\vec{\rho}, j, \tau) \geq u(\vec{\rho}, j, \tau')$ for $\tau \geq \tau'$. Finally, there is a loss function $\ell : \cP \times \Pi \to \Z_{\geq 0}$ which measures the minimum budget the algorithm requires to find a solution. Specifically, $\ell(\vec{\rho}, j) = \infty$ if $u(\vec{\rho}, j, \tau) = 0$ for all $\tau$, and otherwise, $\ell(\vec{\rho}, j) = \argmin\left\{ \tau: u(\vec{\rho}, j, \tau) = 1\right\}$. In Section~\ref{sec:examples}, we provide several examples of this problem definition instantiated for combinatorial problems.

The distribution $\Gamma$ over problem instances is unknown, so we use samples from $\Gamma$ to find a parameter vector $\hat{\vec{\rho}} \in \cP$ with small expected loss. Ideally, we could guarantee that \begin{equation}\E_{j \sim \Gamma}\left[\ell\left(\hat{\vec{\rho}}, j\right)\right] \leq (1+\epsilon)\inf_{\vec{\rho} \in \cP}\left\{\E_{j \sim \Gamma}\left[\ell\left(\vec{\rho}, j\right)\right]\right\}.\label{eq:ideal}\end{equation} Unfortunately, this ideal goal is impossible to achieve with a finite number of samples, even in the extremely simple case where there are only two configurations, as illustrated below.

\begin{example}\label{ex:Weisz}[\citet{Weisz19:CapsAndRuns}]
	Let $\cP = \{1, 2\}$ be a set of two configurations.
	Suppose that the loss of the first configuration is 2 for all problem instances: $\ell(1, j) = 2$ for all $j \in \Pi$. Meanwhile, suppose that $\ell(2, j) = \infty$ with probability $\delta$ for some $\delta \in (0,1)$ and $\ell(2, j) = 1$ with probability $1-\delta$. In this case, $\E_{j \sim \Gamma}[\ell(1,j)] = 2$ and $\E_{j \sim \Gamma}[\ell(2, j)] = \infty$.
	In order for any algorithm to verify that the first configuration's expected loss is substantially better than the second's, it must sample at least one problem instance $j$ such that $\ell(2,j) = \infty$. Therefore, it must sample $\Omega(1/\delta)$ problem instances, a lower bound that approaches infinity as $\delta$ shrinks. As a result, it is impossible to give a finite bound on the number of samples sufficient to find a parameter $\hat{\vec{\rho}}$ that satisfies Equation~\eqref{eq:ideal}.
\end{example}

The obstacle that this example exposes is that some configurations might have an enormous loss on a few rare problem instances.
To deal with this impossibility result, 
\citet{Weisz18:LEAPSANDBOUNDS,Weisz19:CapsAndRuns}, building off of work by \citet{Kleinberg17:Efficiency}, propose a relaxed notion of approximate optimality. To describe this relaxation, we introduce the following notation. Given $\delta \in (0,1)$ and a parameter vector $\vec{\rho} \in \cP$, let $t_{\delta}(\vec{\rho})$ be the largest cutoff $\tau \in \Z_{\geq 0}$ such that the probability $\ell(\vec{\rho}, j)$ is greater than $\tau$ is at least $\delta$. Mathematically, $t_{\delta}(\vec{\rho}) = \argmax_{\tau \in \Z}\left\{\Pr_{j \sim \Gamma}[\ell(\vec{\rho}, j) \geq \tau] \geq \delta\right\}$. The value $t_{\delta}(\vec{\rho})$ can be thought of as the beginning of the loss function's ``$\delta$-tail.''
\begin{figure*}
	\centering
	\includegraphics[width=\textwidth]{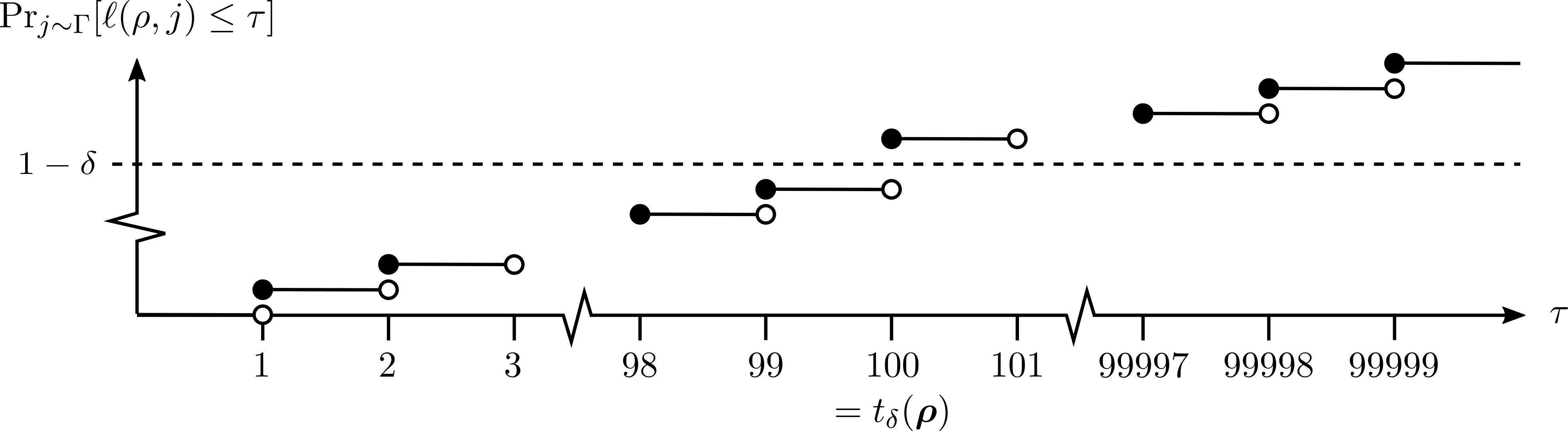}
	\caption{Fix a parameter vector $\vec{\rho}$. The figure is a hypothetical illustration of the cumulative density function of $\ell(\vec{\rho}, j)$ when $j$ is sampled from $\Gamma$. For each value $\tau$ along the $x$-axis, the solid line equals $\Pr_{j \sim \Gamma}\left[\ell(\vec{\rho}, j) \leq \tau\right]$. The dotted line equals the constant function $1-\delta$. Since 100 is the largest integer such that $\Pr_{j \sim \Gamma}\left[\ell(\vec{\rho}, j) \geq 100\right] \geq \delta$, we have that $t_{\delta}(\vec{\rho}) = 100.$}\label{fig:delta}
\end{figure*}
We illustrate the definition of $t_{\delta}(\vec{\rho})$ in Figure~\ref{fig:delta}.
We now define the relaxed notion of approximate optimality by \citet{Weisz18:LEAPSANDBOUNDS}.
\begin{definition}[$\left(\epsilon, \delta, \cP\right)$-optimality]\label{def:opt}
	A parameter vector $\hat{\vec{\rho}}$ is \emph{$\left(\epsilon, \delta, \cP\right)$-optimal} if \[\E_{j \sim \Gamma}\left[\min\left\{\ell\left(\hat{\vec{\rho}}, j\right), t_{\delta}\left(\hat{\vec{\rho}}\right)\right\}\right] \leq (1+\epsilon)\inf_{\vec{\rho} \in \cP} \left\{\E_{j \sim \Gamma}\left[\min\left\{\ell\left(\vec{\rho}, j\right), t_{\delta/2}(\vec{\rho})\right\}\right]\right\}.\]
\end{definition}
In other words, a parameter vector $\hat{\vec{\rho}}$ is \emph{$\left(\epsilon, \delta, \cP\right)$-optimal} if its $\delta$-capped expected loss is within a $(1 + \epsilon)$-factor of the optimal $\delta/2$-capped expected loss.\footnote{The fraction $\delta/2$ can be replaced with any $c\delta$ for $c \in (0,1)$. Ideally, we would replace $\delta/2$ with $\delta$, but the resulting property would be impossible to verify with high probability~\citep{Weisz18:LEAPSANDBOUNDS}.} To condense notation, we write $OPT_{c\delta} := \inf_{\vec{\rho} \in \cP} \left\{\E_{j \sim \Gamma}\left[\min\left\{\ell\left(\vec{\rho}, j\right), t_{c\delta}(\vec{\rho})\right\}\right]\right\}$. If an algorithm returns an $\left(\epsilon, \delta, \bar{\cP}\right)$-optimal parameter from within a finite set $\bar{\cP}$, we call it a \emph{configuration algorithm for finite parameter spaces}.
\citet{Weisz19:CapsAndRuns} provide one such algorithm, {\sc CapsAndRuns}.

\subsection{Example applications}\label{sec:examples}

In this section, we provide several instantiations of our problem definition in combinatorial domains.

\paragraph{Tree search.}
Tree search algorithms, such as branch-and-bound, are the most widely-used tools for solving combinatorial problems, such as (mixed) integer programs and constraint satisfaction problems. These algorithms recursively partition the search space to find an optimal solution, organizing this partition as a tree. Commercial solvers such as CPLEX, which use tree search under the hood, come with hundreds of tunable parameters.
Researchers have developed machine learning algorithms for tuning these parameters~\citep{Hutter09:Paramils,Dickerson13:Throwing,He14:Learning,Balafrej15:Multi,Khalil16:Learning,Khalil17:Learning,Kruber17:Learning,Xia18:Learning,Balcan18:Learning}.
Given parameters $\vec{\rho}$ and a problem instance $j$, we might define the budget $\tau$ to cap the size of the tree the algorithm builds. In that case, the utility function is defined such that $u(\vec{\rho}, j, \tau) = 1$ if and only if the algorithm terminates, having found the optimal solution, after building a tree of size $\tau$. The loss $\ell(\vec{\rho}, j)$ equals the size of the entire tree built by the algorithm parameterized by $\vec{\rho}$ given the instance $j$ as input.

\paragraph{Clustering.}
Given a set of datapoints and the distances between each point, the goal in clustering is to partition the points into subsets so that points within any set are ``similar.'' 
Clustering algorithms are used to group proteins by function, classify images by subject, and myriad other applications.
Typically, the quality of a clustering is measured by an objective function, such as the classic $k$-means, $k$-median, or $k$-center objectives. Unfortunately, it is NP-hard to determine the clustering that minimizes any of these objectives. As a result, researchers have developed a wealth of approximation and heuristic clustering algorithms. However, no one algorithm is optimal across all applications.

\citet{Balcan17:Learning} provide sample complexity guarantees for clustering algorithm configuration. Each problem instance is a set of datapoints and there is a distribution over clustering problem instances. They analyze several infinite classes of clustering algorithms. Each of these algorithms begins with a linkage-based step and concludes with a dynamic programming step.
The linkage-based routine constructs a hierarchical tree of clusters.
At the beginning of the process, each datapoint is in a cluster of its own. The algorithm sequentially merges the clusters into larger clusters until all elements are in the same cluster. There are many ways to build this tree: merge the clusters that are closest in terms of their two closest points (\emph{single-linkage}), their two farthest points (\emph{complete-linkage}), or on average over all pairs of points (\emph{average-linkage}). These linkage procedures are commonly used in practice~\citep{Awasthi17:Local,Saeed03:Software,White10:Alignment}
and come with theoretical guarantees.
\citet{Balcan17:Learning} study an infinite parameterization, \emph{$\rho$-linkage}, that interpolates between single-, average-, and complete-linkage.
After building the cluster tree, the dynamic programming step returns the pruning of this tree that minimizes a fixed objective function, such as the $k$-means, $k$-median, or $k$-center objectives.

Building the full hierarchy is expensive because the best-known algorithm's runtime is $O(n^2\log n)$, where $n$ is the number of datapoints~
\citep{Manning10:Introduction}. It is not always necessary, however, 
to build the entire tree: the algorithm can preemptively terminate the linkage step 
after $\tau$ merges, then use dynamic programming to recover the best pruning of 
the cluster forest. We refer to this variation as \emph{$\tau$-capped $\rho$-linkage}. To evaluate the resulting clustering, we assume there is a cost function $c: \cP \times \Pi \times \Z \to \R$ where $c(\rho, j, \tau)$ measures the
quality of the clustering $\tau$-capped $\rho$-linkage returns, given the instance $j$ as input. We assume there is a 
threshold $\theta_j$ where the clustering is admissible if and 
only if $c(\rho, j, \tau) \leq \theta_j$, which means the utility function is defined as $u(\rho, j, \tau) = \textbf{1}_{\left\{c(\rho, j, \tau) \leq \theta_j\right\}}$. For example, $c(\rho, j, \tau)$ might measure the clustering's $k$-means objective value, and $\theta_j$ might equal the optimal $k$-means objective value (obtained only for the training instances via an expensive computation) plus an error term.

\section{Data-dependent discretizations of infinite parameter spaces}
We begin this section by proving an intuitive fact: given a finite subset $\bar{\cP} \subset \cP$ of parameters that contains at least one ``sufficiently good'' parameter, a configuration algorithm for finite parameter spaces, such as {\sc CapsAndRuns}~\citep{Weisz19:CapsAndRuns}, returns a parameter that's nearly optimal over the infinite set $\cP$. Therefore, our goal is to provide an algorithm that takes as input an infinite parameter space and returns a finite subset that contains at least one good parameter. A bit more formally, a parameter is ``sufficiently good'' if its $\delta/2$-capped expected loss is within a $\sqrt{1+\epsilon}$-factor of $OPT_{\delta/4}$.
We say a finite parameter set $\bar{\cP}$ is an \emph{$\left(\epsilon, \delta\right)$-optimal subset} if it contains a good parameter.

\begin{definition}[$\left(\epsilon, \delta\right)$-optimal subset]\label{def:opt_subset}
	A finite set $\bar{\cP} \subset \cP$ is an \emph{$\left(\epsilon, \delta\right)$-optimal subset} if there is a vector $\hat{\vec{\rho}} \in \bar{\cP}$ such that $\E_{j \sim \Gamma}\left[\min\left\{\ell\left(\hat{\vec{\rho}}, j\right), t_{\delta/2}\left(\hat{\vec{\rho}}\right)\right\}\right] \leq \sqrt{1+\epsilon}\cdot OPT_{\delta/4}.$
\end{definition}

We now prove that given an $(\epsilon, \delta)$-optimal subset $\bar{\cP} \subset \cP$, a configuration algorithm for finite parameter spaces returns a nearly optimal parameter from the infinite space $\cP$.

\begin{theorem}\label{thm:reduction}
	Let $\bar{\cP} \subset \cP$ be an $(\epsilon, \delta)$-optimal subset and let $\epsilon' = \sqrt{1 + \epsilon} - 1$. Suppose $\hat{\vec{\rho}} \in \bar{\cP}$ is $\left(\epsilon', \delta, \bar{\cP}\right)$-optimal. Then $\E_{j \sim \Gamma}\left[\min\left\{\ell\left(\hat{\vec{\rho}}, j\right), t_{\delta}\left(\hat{\vec{\rho}}\right)\right\}\right] \leq (1 + \epsilon) \cdot OPT_{\delta/4}$.
\end{theorem}

\begin{proof}
	Since the parameter $\hat{\vec{\rho}}$ is $\left(\epsilon', \delta, \bar{\cP}\right)$-optimal, we know that $\E_{j \sim \Gamma}\left[\min\left\{\ell\left(\hat{\vec{\rho}}, j\right), t_{\delta}\left(\hat{\vec{\rho}}\right)\right\}\right] \leq \sqrt{1 + \epsilon} \cdot \min_{\vec{\rho} \in \bar{\cP}} \left\{\E_{j \sim \Gamma}\left[\min\left\{\ell\left(\vec{\rho}, j\right), t_{\delta/2}(\vec{\rho})\right\}\right]\right\}$. (We use a minimum instead of an infimum because $\bar{\cP}$ is a finite set by Definition~\ref{def:opt_subset}.)
	The set $\bar{\cP}$ is an $(\epsilon, \delta)$-optimal subset of the parameter space $\cP$, so there exists a parameter vector $\vec{\rho}' \in \bar{\cP}$ such that $\E_{j \sim \Gamma}\left[\min\left\{\ell\left(\vec{\rho}', j\right), t_{\delta/2}\left(\vec{\rho}'\right)\right\}\right] \leq \sqrt{1+\epsilon}\cdot OPT_{\delta/4}$. Therefore, \begin{align*}\E_{j \sim \Gamma}\left[\min\left\{\ell\left(\hat{\vec{\rho}}, j\right), t_{\delta}\left(\hat{\vec{\rho}}\right)\right\}\right] &\leq \sqrt{1 + \epsilon} \cdot \min_{\vec{\rho} \in \bar{\cP}} \left\{\E_{j \sim \Gamma}\left[\min\left\{\ell\left(\vec{\rho}, j\right), t_{\delta/2}(\vec{\rho})\right\}\right]\right\}\\
		&\leq \sqrt{1 + \epsilon} \cdot \E_{j \sim \Gamma}\left[\min\left\{\ell\left(\vec{\rho}', j\right), t_{\delta/2}\left(\vec{\rho}'\right)\right\}\right]\\
		&\leq (1 + \epsilon) \cdot OPT_{\delta/4},\end{align*} so the theorem statement holds.
\end{proof}

\section{Our main result: Algorithm for learning $(\epsilon, \delta)$-optimal subsets}\label{sec:algorithm}
We present an algorithm for learning $(\epsilon,\delta)$-optimal subsets for configuration problems that satisfy a simple, yet ubiquitous structure: for any problem instance $j$, the loss function $\ell(\cdot, j)$ is piecewise constant. This structure has been observed throughout a diverse array of configuration problems ranging from clustering to integer programming~\citep{Balcan17:Learning,Balcan18:Learning}. More formally, this structure holds if for any problem instance $j \in \Pi$ and cap $\tau \in \Z_{\geq 0}$, there is a finite partition of the parameter space $\cP$ such that in any one region $R$ of this partition, for all pairs of parameter vectors $\vec{\rho}, \vec{\rho}' \in R$, $\min\left\{\ell(\vec{\rho}, j), \tau\right\} = \min\left\{\ell(\vec{\rho}', j), \tau\right\}$.

To exploit this piecewise-constant structure, we require access to a function $\gp$ that takes as input a set $\sample$ of problem instances and an integer $\tau$ and returns this partition of the parameters. Namely, it returns a set of tuples $\left(\cP_1, z_1, \vec{\tau}_1\right), \dots, \left(\cP_k, z_k, \vec{\tau}_k\right) \in 2^{\cP} \times [0,1] \times \Z^{|\sample|}$ such that:
\begin{enumerate}
	\item The sets $\cP_1, \dots, \cP_k$ make up a partition of $\cP$.
	\item For all subsets $\cP_i$ and vectors $\vec{\rho}, \vec{\vec{\rho}'} \in \cP_i$, $\frac{1}{|\sample|} \sum_{j \in \sample} \textbf{1}_{\left\{\ell(\vec{\rho}, j) \leq \tau\right\}} = \frac{1}{|\sample|} \sum_{j \in \sample} \textbf{1}_{\left\{\ell(\vec{\rho}', j) \leq \tau\right\}} = z_i.$
	\item For all subsets $\cP_i$, all $\vec{\rho}, \vec{\vec{\rho}'} \in \cP_i$, and all $j \in \sample$, $\min\left\{\ell(\vec{\rho}, j), \tau\right\} = \min\left\{\ell(\vec{\rho}', j), \tau\right\} = \tau_i[j].$
\end{enumerate}
We assume the number of tuples $\gp$ returns is monotone: if $\tau \leq \tau'$, then $\left|\gp(\sample, \tau)\right| \leq \left|\gp(\sample, \tau')\right|$ and if $\sample \subseteq \sample'$, then $\left|\gp(\sample, \tau)\right| \leq \left|\gp(\sample', \tau)\right|$.

As we describe in Appendix~\ref{app:related}, results from prior research imply guidance for implementing $\gp$ in the contexts of clustering and integer programming. For example, in the clustering application we describe in Section~\ref{sec:examples}, the distribution $\Gamma$ is over clustering instances. Suppose $n$ is an upper bound on the number of points in each instance. \citet{Balcan17:Learning} prove that for any set $\sample$ of samples and any cap $\tau$, in the worst case, $\left|\gp(\sample, \tau)\right| = O\left(|\sample|n^8\right)$, though empirically, $\left|\gp(\sample, \tau)\right|$ is often several orders of magnitude smaller~\citep{Balcan19:Learning}. \citet{Balcan17:Learning} and \citet{Balcan19:Learning} provide guidance for implementing $\gp$.

\begin{algorithm}[t]
	\caption{Algorithm for learning $(\epsilon,\delta)$-optimal subsets}\label{alg:findSubset}
	\begin{algorithmic}[1]
		\Require Parameters $\delta, \zeta \in (0,1), \epsilon > 0$.
		\State Set $\eta \leftarrow \min\left\{\frac{1}{8}\left(\sqrt[4]{1 + \epsilon} - 1\right), \frac{1}{9}\right\}$, $t \leftarrow 1$, $T \leftarrow \infty$, and $\cG \leftarrow \emptyset$.
		\While{$2^{t-3}\delta < T$}\label{step:outer_while}
		\State Set $\sample_t \leftarrow \{j\}$, where $j \sim \Gamma$.
		\While{$\eta\delta < \sqrt{\frac{2d \ln \left|\gp\left(\sample_t, 2^t\right)\right|}{\left|\sample_t\right|}} + \sqrt{\frac{8}{\left|\sample_t\right|}\ln \frac{8\left(2^t \left|\sample_t\right| t\right)^2}{\zeta}}$}
		Draw $j \sim \Gamma$ and add $j$ to $\sample_t$.\label{step:growS}
		\EndWhile
		\State Compute tuples $\left(\cP_1, z_1, \vec{\tau}_1\right), \dots, \left(\cP_k, z_k, \vec{\tau}_k\right) \leftarrow \gp\left(\sample_t, 2^t\right)$.\label{step:partition}
		\For {$i \in \{1, \dots, k\}$ with $z_i \geq 1 - 3\delta/8$}\label{step:delta}
		\State Set $\cG \leftarrow \cG \cup \left\{\cP_i\right\}$.\label{step:updateG}
		\State Sort the elements of $\vec{\tau}_i$: $\tau_1 \leq \cdots \leq \tau_{\left|\sample_t\right|}$.\label{step:sort}
		\State Set $T' \leftarrow \frac{1}{\left|\sample_t\right|}\sum_{m = 1}^{\left|\sample_t\right|}\min \left\{\tau_m, \tau_{\lfloor \left|\sample_t\right|\left(1 - 3\delta/8\right)\rfloor}\right\}.$\label{step:estimate}
		\If {$T' < T$}
		Set $T \leftarrow T'.$\label{step:updateT}
		\EndIf
		\EndFor
		\State $t \leftarrow t + 1$.
		\EndWhile
		\State For each set $\cP' \in \cG$, select a vector $\vec{\rho}_{\cP'} \in \cP'$.\label{step:arbitrary}
		\Ensure The $(\epsilon, \delta)$-optimal set $\left\{\vec{\rho}_{\cP'} \mid \cP' \in \cG\right\}$.
	\end{algorithmic}
\end{algorithm}

\paragraph{High-level description of algorithm.} We now describe our algorithm for learning $(\epsilon,\delta)$-optimal subsets. See Algorithm~\ref{alg:findSubset} for the pseudocode. The algorithm maintains a variable $T$, initially set to $\infty$, which roughly represents an upper confidence bound on $OPT_{\delta/4}$. It also maintains a set $\cG$ of parameters which the algorithm believes might be nearly optimal. The algorithm begins by aggressively capping the maximum loss $\ell$ it computes by 1. At the beginning of each round, the algorithm doubles this cap until the cap grows sufficiently large compared to the upper confidence bound $T$. At that point, the algorithm terminates. On each round $t$, the algorithm draws a set $\sample_t$ of samples (Step~\ref{step:growS}) that is just large enough to estimate the expected $2^t$-capped loss $\E_{j \sim \Gamma}\left[\min\left\{\ell(\vec{\rho}, j), 2^t\right\}\right]$ for every parameter $\vec{\rho} \in \cP$. The number of samples it draws is a data-dependent quantity that depends on \emph{empirical Rademacher complexity}~\citep{Koltchinskii01:Rademacher,Bartlett02:Rademacher}.

Next, the algorithm evaluates the function $\gp\left(\sample_t, 2^t\right)$ to obtain the tuples \[\left(\cP_1, z_1, \vec{\tau}_1\right), \dots, \left(\cP_k, z_k, \vec{\tau}_k\right) \in 2^{\cP} \times [0,1] \times \Z^{|\sample_t|}.\] By definition of this function, for all subsets $\cP_i$ and parameter vector pairs $\vec{\rho}, \vec{\rho}' \in \cP_i$,
the fraction of instances $j \in \sample_t$ with $\ell(\vec{\rho}, j) \leq 2^t$ is equal to the fraction of instances $j \in \sample_t$ with $\ell(\vec{\rho}', j) \leq 2^t$.
In other words, $\frac{1}{|\sample_t|} \sum_{j \in \sample_t} \textbf{1}_{\left\{\ell(\vec{\rho}, j) \leq 2^t\right\}} = \frac{1}{|\sample_t|} \sum_{j \in \sample_t} \textbf{1}_{\left\{\ell(\vec{\rho}', j) \leq 2^t\right\}} = z_i.$ If this fraction is sufficiently high (at least $1 - 3\delta/8$), the algorithm adds $\cP_i$ to the set of good parameters $\cG$ (Step~\ref{step:updateG}). The algorithm estimates the $\delta/4$-capped expected loss of the parameters contained $\cP_i$, and if this estimate is smaller than the current upper confidence bound $T$ on $OPT_{\delta/4}$, it updates $T$ accordingly (Steps~\ref{step:sort} through \ref{step:updateT}). Once the cap $2^t$ has grown sufficiently large compared to the upper confidence bound $T$, the algorithm returns an arbitrary parmeter from each set in $\cG$.

\paragraph{Algorithm analysis.} We now provide guarantees on Algorithm~\ref{alg:findSubset}'s performance.
We denote the values of $t$ and $T$ at termination by $\bar{t}$ and $\bar{T}$, and we denote the state of the set $\cG$ at termination by $\bar{\cG}$. For each set $\cP' \in \bar{\cG}$, we use the notation $\tau_{\cP'}$ to denote the value $\tau_{\lfloor |\sample_t|\left(1 - 3\delta/8\right)\rfloor}$ in Step~\ref{step:estimate} during the iteration $t$ that $\cP'$ is added to $\cG$.

\begin{theorem}\label{thm:main}
	With probability $1 - \zeta$, the following conditions hold, with \[\eta = \min\left\{\frac{\sqrt[4]{1 + \epsilon} - 1}{8}, \frac{1}{9}\right\} \quad \text{and} \quad c = \frac{16\sqrt[4]{1+\epsilon}}{\delta}:\] 
	\begin{enumerate}
		\item Algorithm~\ref{alg:findSubset} terminates after $\bar{t} = O\left(\log\left(c\cdot OPT_{\delta/4}\right)\right)$ iterations.
		\item Algorithm~\ref{alg:findSubset} returns an $(\epsilon,\delta)$-optimal set of parameters of size at most \[\sum_{t = 1}^{\bar{t}}\left|\gp\left(\sample_t, c \cdot OPT_{\delta/4}\right)\right|.\]
		\item The sample complexity on round $t \in \left[\bar{t}\right]$, $\left|\sample_t\right|$, is
	\end{enumerate}
	\[\tilde O\left(\frac{d \ln \left|\gp\left(\sample_{t}, c \cdot OPT_{\delta/4}\right)\right| +  c\cdot OPT_{\delta/4}}{\eta^2\delta^2}\right).\]
\end{theorem}
\begin{proof}
	We split the proof into separate lemmas. Lemma~\ref{lem:bound_bar_t} proves Part 1. Lemma~\ref{lem:opt_subset} as well as Lemma~\ref{lem:size_opt_subset} in Appendix~\ref{app:algorithm} prove Part 2. Finally, Part 3 follows from classic results in learning theory on Rademacher complexity. In particular, it follows from an inversion of the inequality in Step~\ref{step:growS} and the fact that $2^t \leq 2^{\bar{t}} \leq c \cdot OPT_{\delta/4}$, as we prove in Lemma~\ref{lem:bound_bar_t}.
\end{proof}

Theorem~\ref{thm:main} hinges on the assumption that the samples $\sample_1, \dots, \sample_{\bar{t}}$ Algorithm~\ref{alg:findSubset} draws in Step~\ref{step:growS} are sufficiently representative of the distribution $\Gamma$, formalized as follows:

\begin{definition}[$\zeta$-representative run]
	For each round $t \in \left[\bar{t}\right]$, denote the samples in $\sample_t$ as $\sample_t = \left\{j_i^{(t)} : i \in \left[\left|\sample_t\right|\right]\right\}$. We say that Algorithm~\ref{alg:findSubset} has a \emph{$\zeta$-representative run} if for all rounds $t \in \left[\bar{t}\right]$, all integers $b \in \left[\left|\sample_t\right|\right]$, all caps $\tau \in \Z_{\geq 0}$, and all parameters $\vec{\rho} \in \cP$, the following conditions hold:
	\begin{enumerate}
		\item The average number of instances $j_i^{(t)} \in \left\{j_1^{(t)}, \dots, j_b^{(t)}\right\}$ with loss smaller than $\tau$ nearly matches the probability that $\ell(\vec{\rho}, j) \leq \tau$:
		$\left|\frac{1}{b} \sum_{i = 1}^b \textbf{1}_{\left\{\ell\left(\vec{\rho}, j_i^{(t)}\right) \leq \tau\right\}} - \Pr_{j \sim \Gamma}\left[\ell(\vec{\rho}, j) \leq \tau\right]\right| \leq \gamma(t, b, \tau)$, and
		\item The average $\tau$-capped loss of the instances $j_1^{(t)}, \dots, j_b^{(t)}$ nearly matches the expected $\tau$-capped loss:
\[\left|\frac{1}{b} \sum_{i = 1}^b \min\left\{\ell\left(\vec{\rho}, j_i^{(t)}\right), \tau\right\} - \E_{j \sim \Gamma}\left[\min\left\{\ell(\vec{\rho}, j), \tau\right\}\right]\right|
		\leq \tau\cdot \gamma(t, b, \tau),\] where \[\gamma(t, b, \tau) = \sqrt{\frac{2d \ln \left|\gp\left(\left\{j_1^{(t)}, \dots, j_b^{(t)}\right\}, \tau\right)\right|}{b}} + 2\sqrt{\frac{2}{b}\ln \frac{8\left(\tau b t\right)^2}{\zeta}}.\]
		\end{enumerate}
\end{definition}

In Step~\ref{step:growS}, we ensure $\sample_t$ is large enough that Algorithm~\ref{alg:findSubset} has a $\zeta$-representative run with probability $1-\zeta$.
\begin{lemma}\label{lem:representative}
	With probability $1-\zeta$, Algorithm~\ref{alg:findSubset} has a $\zeta$-representative run.
\end{lemma}

Lemma~\ref{lem:representative} is a corollary of a Rademacher complexity analysis which we include in Appendix~\ref{app:algorithm} (Lemma~\ref{lem:rad}). Intuitively, there are only $\left|\gp(\sample, \tau)\right|$ algorithms with varying $\tau$-capped losses over any set of samples $\sample$. We can therefore invoke Massart's finite lemma~\citep{Massart00:Some}, which guarantees that each set $\sample_t$ is sufficiently large to ensure that Algorithm~\ref{alg:findSubset} indeed has a $\zeta$-representative run.
The remainder of our analysis will assume that Algorithm~\ref{alg:findSubset} has a $\zeta$-representative run.

\paragraph{Number of iterations until termination.}
We begin with a proof sketch of the first part of Theorem~\ref{thm:main}. The full proof is in Appendix~\ref{app:algorithm}.

\begin{restatable}{lemma}{BoundBart}\label{lem:bound_bar_t}
	Suppose Algorithm~\ref{alg:findSubset} has a $\zeta$-representative run. Then $2^{\bar{t}} \leq \frac{16}{\delta}\sqrt[4]{1+\epsilon} \cdot OPT_{\delta/4}$.
\end{restatable}

\begin{proof}[Proof sketch]
	By definition of $OPT_{\delta/4}$, for every $\gamma > 0$, there exists a vector $\vec{\rho}^* \in \cP$ 
	whose $\delta/4$-capped expected loss is within a $\gamma$-factor of optimal:
	$\E_{j \sim \Gamma}\left[\min\left\{\ell\left(\vec{\rho}^*, j\right), t_{\delta/4}\left(\vec{\rho}^*\right)\right\}\right] \leq OPT_{\delta/4} + \gamma.$ We prove this vector's $\delta/4$-capped expected loss bounds $\bar{t}$: \begin{equation}2^{\bar{t}} \leq \frac{16\sqrt[4]{1+\epsilon}}{\delta} \cdot \E_{j \sim \Gamma}\left[\min\left\{\ell\left(\vec{\rho}^*, j\right), t_{\delta/4}\left(\vec{\rho}^*\right)\right\}\right].\label{eq:ub_bart_sketch}\end{equation} This implies the lemma statement holds.
	We split the proof of Inequality~\eqref{eq:ub_bart_sketch} into two cases: one where the vector $\vec{\rho}^*$ is contained within a set $\cP' \in \bar{\cG}$, and the other where it is not. In the latter case, Lemma~\ref{lem:bad_param_lb} bounds $2^{\bar{t}}$ by $\frac{8}{\delta} \cdot \E_{j \sim \Gamma}\left[\min\left\{\ell\left(\vec{\rho}^*, j\right), t_{\delta/4}\left(\vec{\rho}^*\right)\right\}\right]$, which implies Inequality~\eqref{eq:ub_bart_sketch} holds. We leave the other case to the full proof in Appendix~\ref{app:algorithm}.
\end{proof}

In the next lemma, we prove the upper bound on $2^{\bar{t}}$ that we use in Lemma~\ref{lem:bound_bar_t}. The full proof is in Appendix~\ref{app:algorithm}.

\begin{restatable}{lemma}{BadParamLB}\label{lem:bad_param_lb}
	Suppose Algorithm~\ref{alg:findSubset} has a $\zeta$-representative run. For any parameter vector $\vec{\rho} \not\in \bigcup_{\cP' \in \bar{\cG}} \cP'$, $2^{\bar{t}} \leq \frac{8}{\delta} \cdot \E_{j \sim \Gamma}\left[\min\left\{\ell(\vec{\rho}, j), t_{\delta/4}(\vec{\rho})\right\}\right].$
\end{restatable}

\begin{proof}[Proof sketch]
	The last round that Algorithm~\ref{alg:findSubset} adds any subset to the set $\cG$ is round $\bar{t}-1$. For ease of notation, let $\bar{\sample} = \sample_{\bar{t}-1}$. Since $\vec{\rho}$ is not an element of any set in $\cG$, the cap $2^{\bar{t}-1}$ must be too small compared to the average loss of the parameter $\vec{\rho}$. Specifically, it must be that $\frac{1}{\left|\bar{\sample}\right|} \sum_{j \in \bar{\sample}} \textbf{1}_{\left\{\ell\left(\vec{\rho}, j\right) \leq 2^{\bar{t}-1}\right\}} < 1 - \frac{3\delta}{8}$.
	Since Algorithm~\ref{alg:findSubset} had a $\zeta$-representative run,
	the probability the loss of $\vec{\rho}$ is smaller than $2^{\bar{t}-1}$ converges to the fraction of samples with loss smaller than $2^{\bar{t}-1}$: $\Pr_{j \sim \Gamma}\left[\ell(\vec{\rho}, j)\leq 2^{\bar{t}-1}\right] \leq \frac{1}{\left|\bar{\sample}\right|} \sum_{j \in \bar{\sample}} \textbf{1}_{\left\{\ell\left(\vec{\rho}, j\right) \leq 2^{\bar{t}-1}\right\}} + \eta\delta,$ so
	$\Pr_{j \sim \Gamma}\left[\ell(\vec{\rho}, j) \leq 2^{\bar{t}-1}\right] < 1 - (3/8 - \eta)\delta$. Since $\eta \leq 1/9$, it must be that $\Pr_{j \sim \Gamma}\left[\ell(\vec{\rho}, j) \geq 2^{\bar{t}-1}\right] \geq \delta/4$, so by definition of $t_{\delta/4}(\vec{\rho})$, we have that $2^{\bar{t}-1} \leq t_{\delta/4}(\vec{\rho})$. Therefore, $\E_{j \sim \Gamma}\left[\min\left\{\ell(\vec{\rho}, j), t_{\delta/4}(\vec{\rho})\right\}\right] \geq \frac{\delta}{4} \cdot t_{\delta/4}(\vec{\rho}) \geq 2^{{\bar{t}}-3}\delta$.
\end{proof}

\paragraph{Optimality of Algorithm~\ref{alg:findSubset}'s output.}
Next, we provide a proof sketch of the second part of Theorem~\ref{thm:main}, which guarantees that Algorithm~\ref{alg:findSubset} returns an $(\epsilon, \delta)$-optimal subset. The full proof is in Appendix~\ref{app:algorithm}. For each set $\cP' \in \bar{\cG}$, $\tau_{\cP'}$ denotes the value of $\tau_{\lfloor |\sample_t|\left(1 - 3\delta/8\right)\rfloor}$ in Step~\ref{step:estimate} of Algorithm~\ref{alg:findSubset} during the iteration $t$ that $\cP'$ is added to $\cG$.

\begin{restatable}{lemma}{optSubset}\label{lem:opt_subset}
	If Algorithm~\ref{alg:findSubset} has a $\zeta$-representative run, it returns an $(\epsilon, \delta)$-optimal subset.
\end{restatable}

\begin{proof}[Proof sketch]
	By definition of $OPT_{\delta/4}$, for every $\gamma > 0$, there is a vector $\vec{\rho}^* \in \cP$ such that $\E_{j \sim \Gamma}\left[\min\left\{\ell(\vec{\rho}^*, j), t_{\delta/4}(\vec{\rho}^*)\right\}\right] \leq OPT_{\delta/4} + \gamma.$ Let $\cP^*$ be the output of Algorithm~\ref{alg:findSubset}. We claim there exists a parameter $\vec{\rho}' \in \cP^*$ such that \begin{equation}\E_{j \sim \Gamma}\left[\min\left\{\ell\left(\vec{\rho}', j\right), t_{\delta/2}\left(\vec{\rho}'\right)\right\}\right] \leq \sqrt{1 + \epsilon} \cdot \E_{j \sim \Gamma}\left[\min\left\{\ell(\vec{\rho}^*, j), t_{\delta/4}(\vec{\rho}^*)\right\}\right],\label{eq:inSet}\end{equation} which implies the lemma statement.
	There are two cases: either the vector $\vec{\rho}^*$ is contained within a set $\cP' \in \bar{\cG}$, or it is not. In this sketch, we analyze the latter case.
	
	By Lemma~\ref{lem:bad_param_lb}, we know that $\E_{j \sim \Gamma}\left[\min\left\{\ell\left(\vec{\rho}^*, j\right), t_{\delta/4}(\vec{\rho}^*)\right\}\right] \geq 2^{\bar{t}-3}\delta.$ When Algorithm~\ref{alg:findSubset} terminates, $2^{\bar{t}-3}\delta$ is greater than the upper confidence bound $\bar{T}$, which means that \[\E_{j \sim \Gamma}\left[\min\left\{\ell(\vec{\rho}^*, j), t_{\delta/4}(\vec{\rho}^*)\right\}\right] > \bar{T}.\]
	We next derive a lower bound on $\bar{T}$: we prove that there exists a set $\cP' \in \bar{\cG}$ and parameter vector $\vec{\rho} \in \cP'$ such that $\sqrt[4]{1+\epsilon} \cdot \bar{T} \geq \E_{j \sim \Gamma}\left[\min \left\{\ell\left(\vec{\rho}, j\right), \tau_{\cP'}\right\}\right]$.
	Our upper bound on $\bar{T}$ implies that $\E_{j \sim \Gamma}\left[\min \left\{\ell\left(\vec{\rho}, j\right), \tau_{\cP'}\right\}\right] \leq \sqrt[4]{1+\epsilon}\cdot \E_{j \sim \Gamma}\left[\min\left\{\ell\left(\vec{\rho}^*, j\right), t_{\delta/4}(\vec{\rho})\right\}\right]$.
	Finally, we prove that there is a parameter $\vec{\rho}' \in \cP' \cap \cP^*$ whose $\delta/2$-capped expected loss is within a $\sqrt[4]{1+\epsilon}$-factor of the expected loss $\E_{j \sim \Gamma}\left[\min \left\{\ell\left(\vec{\rho}, j\right), \tau_{\cP'}\right\}\right]$. This follows from a proof that $\tau_{\cP'}$ approximates $t_{\delta/4}(\vec{\rho}')$ and the fact that $\vec{\rho}$ and $\vec{\rho'}$ are both elements of $\cP'$.
	Stringing these inequalities together, we prove that Equation~\eqref{eq:inSet} holds.
\end{proof}

The second part of Theorem~\ref{thm:main} also guarantees that the size of the set Algorithm~\ref{alg:findSubset} returns is bounded (see Lemma~\ref{lem:size_opt_subset} in Appendix~\ref{app:algorithm}).
Together, Lemmas~\ref{lem:bound_bar_t} and \ref{lem:opt_subset}, as well as Lemma~\ref{lem:size_opt_subset} in Appendix~\ref{app:algorithm}, bound the number of iterations Algorithm~\ref{alg:findSubset} makes until it returns an $(\epsilon, \delta)$-optimal subset.

\section{Comparison to prior research}\label{sec:comparison}
We now provide comparisons to prior research on algorithm configuration with provable guarantees. Both comparisons revolve around branch-and-bound (B\&B) configuration for integer programming (IP), overviewed in Section~\ref{sec:examples}.

\paragraph{Uniformly sampling configurations.} Prior research provides algorithms for finding nearly-optimal configurations from a finite set~\citep{Kleinberg17:Efficiency,Kleinberg19:Procrastinating,Weisz18:LEAPSANDBOUNDS,Weisz19:CapsAndRuns}.  If the parameter space is infinite and their algorithms optimize over a uniformly-sampled set of $\tilde{\Omega}(1 / \gamma)$ configurations, then the output configuration will be within the top $\gamma$-quantile, with high probability. If the set of good parameters is small, however, the uniform sample might not include any of them. Algorithm configuration problems where the high-performing parameters lie within a small region do exist, as we illustrate in the following theorem.

\begin{theorem}[\citet{Balcan18:Learning}]\label{thm:worst_case}
	For any $\frac{1}{3} < a < b < \frac{1}{2}$ and $n \geq 6$, there are infinitely-many distributions $\Gamma$ over IPs with $n$ variables and a B\&B parameter\footnote{As we describe in Appendix~\ref{app:IP}, $\rho$ controls the variable selection policy. The theorem holds for any node selection policy.} with range $[0,1]$ such that:
	\begin{enumerate}
		\item If $\rho \leq a$, then $\ell(\rho, j) = 2^{(n-5)/4}$ with probability $\frac{1}{2}$ and $\ell(\rho, j)=8$ with probability $\frac{1}{2}$.
		\item If $\rho \in (a,b)$, then $\ell(\rho, j) = 8$ with probability $1$.
		\item If $\rho \geq b$, then $\ell(\rho, j) = 2^{(n-4)/2}$ with probability $\frac{1}{2}$ and $\ell(\rho, j) = 8$ with probability $\frac{1}{2}$.
	\end{enumerate}
	Here, $\ell(\rho,j)$ measures the size of the tree B\&B builds using the parameter $\rho$ on the input integer program $j$.
\end{theorem}

In the above configuration problem, any parameter in the range $(a,b)$ has a loss of 8 with probability 1, which is the minimum possible loss. Any parameter outside of this range has an abysmal expected loss of at least $2^{(n-6)/2}$. In fact, for any $\delta \leq 1/2$, the $\delta$-capped expected loss of any parameter in the range $[0,a]\cup[b, 1]$ is at least $2^{(n-6)/2}$. Therefore, if we uniformly sample a finite set of parameters and optimize over this set using an algorithm for finite parameter spaces~\citep{Kleinberg17:Efficiency,Kleinberg19:Procrastinating,Weisz18:LEAPSANDBOUNDS,Weisz19:CapsAndRuns}, we must ensure that we sample at least one parameter within $(a,b)$. As $a$ and $b$ converge, however, the required number of samples shoots to infinity, as we formalize below. This section's omitted proofs are in Appendix~\ref{app:IP}.
\begin{restatable}{theorem}{sampling}\label{thm:sampling}
	For the B\&B configuration problem in Theorem~\ref{thm:worst_case}, with constant probability over the draw of $m = \left\lfloor 1 / (b-a) \right\rfloor$ parameters $\rho_1, \dots, \rho_m \sim \text{Uniform}[0,1]$, $\left\{\rho_1, \dots, \rho_m\right\} \cap (a,b) = \emptyset$.
\end{restatable}

Meanwhile, Algorithm~\ref{alg:findSubset} quickly terminates, having found an optimal parameter, as we describe below.

\begin{restatable}{theorem}{ourAlg}\label{thm:our_alg}
	For the B\&B configuration problem in Theorem~\ref{thm:worst_case}, Algorithm~\ref{alg:findSubset} terminates after $\tilde{O}(\log 1/\delta)$ iterations, having drawn $\tilde O((\delta\eta)^{-2})$ sample problem instances where \[\eta = \min\left\{\frac{\sqrt[4]{1 + \epsilon} - 1}{8}, \frac{1}{9}\right\}\] and returns a set containing an optimal parameter in $(a,b)$.
\end{restatable}

Similarly, \citet{Balcan18:LearningTheoretic} exemplify clustering configuration problems---which we overview in Section~\ref{sec:examples}---where the optimal parameters lie within an arbitrarily small region, and any other parameter leads to significantly worse performance. As in Theorem~\ref{thm:sampling}, this means a uniform sampling of the parameters will fail to find optimal parameters.

\paragraph{Uniform convergence.}
Prior research has provided uniform convergence sample complexity bounds for algorithm configuration. These guarantees bound the number of samples sufficient to ensure that for any configuration, its average loss over the samples nearly matches its expected loss.

We prove that in the case of B\&B configuration, Algorithm~\ref{alg:findSubset} may use far fewer samples to find a nearly optimal configuration than the best-known uniform convergence sample complexity bound.
\citet{Balcan18:Learning} prove uniform convergence sample complexity guarantees for B\&B configuration. They bound the number of samples sufficient to ensure that for any configuration in their infinite parameter space, the size of the search tree B\&B builds on average over the samples generalizes to the expected size of the tree it builds.
For integer programs over $n$ variables, the best-known sample complexity bound guarantees that $\left(2^n/\epsilon'\right)^2$ samples are sufficient to ensure that the average tree size B\&B builds over the samples is within an additive $\epsilon'$ factor of the expected tree size~\citep{Balcan18:Learning}. Meanwhile, as we describe in Theorem~\ref{thm:our_alg}, there are B\&B configuration problems where our algorithm's sample complexity bound is significantly better: our algorithm finds an optimal parameter using only $\tilde O((\delta\eta)^{-2})$ samples.

\section{Conclusion}
We presented an algorithm that learns a finite set of promising parameters from an infinite parameter space. It can be used to determine the input to a configuration algorithm for finite parameter spaces, or as a tool for compiling an algorithm portfolio. We proved bounds on the number of iterations before our algorithm terminates, its sample complexity, and the size of its output. A strength of our approach is its modularity: it can determine the input to a configuration algorithm for finite parameter spaces without depending on specifics of that algorithm's implementation. There is an inevitable tradeoff, however, between modularity and computational efficiency. In future research, our approach can likely be folded into existing configuration algorithms for finite parameter spaces.

\subsection*{Acknowledgments}
This material is based on work supported by the National Science Foundation under grants IIS-1718457, IIS-1617590, IIS-1618714, IIS-1901403, CCF-1535967, CCF-1910321, SES-1919453, and CCF-1733556, the ARO under award W911NF-17-1-0082, an Amazon Research Award, an AWS Machine Learning Research Award, a Bloomberg Data Science research grant, a fellowship from Carnegie Mellon University’s Center for Machine Learning and Health, and the IBM PhD Fellowship.

\bibliography{../../../../VitercikLibrary.bib}
\bibliographystyle{plainnat}

\appendix

\section{Review of Rademacher complexity}\label{app:RC}
At a high level, Rademacher complexity measures the extent to which a class of functions fit random noise. Intuitively, this measures the richness of a function class because more complex classes should be able to fit random noise better than simple classes. Empirical Rademacher complexity can be measured on the set of samples and implies generalization guarantees that improve based on structure exhibited by the set of samples.
Formally, let $\cF \subseteq [0,H]^{\cX}$ be an abstract function class mapping elements of a domain $\cX$ to the interval $[0,H]$.
Given a set $\sample = \left\{x_1, \dots, x_N\right\} \subseteq \cX$, the \emph{empirical Rademacher complexity} of $\cF$ with respect to  $\sample$ is defined as \[\erad\left(\cF\right) = \E_{\vec{\sigma}}\left[\sup_{f \in \cF} \frac{1}{N} \sum_{i = 1}^N \sigma_i \cdot f\left(x_i\right) \right],\] where $\sigma_i \sim \text{Uniform}\left(\left\{-1,1\right\}\right)$. Let $\dist$ be a distribution over $\cX$. Classic results from learning theory~\citep{Koltchinskii01:Rademacher,Bartlett02:Rademacher} guarantee that
with probability $1-\delta$ over the draw $\sample \sim \dist^N$, for every function $f \in \cF$, \[\left|\frac{1}{N}\sum_{x \in \sample} f(x) - \E_{x \sim \dist}[f(x)]\right| \leq \erad\left(\cF\right) + 2H \sqrt{\frac{2\ln \left(4/\delta\right)}{N}}.\]

\citet{Massart00:Some} proved the following Rademacher complexity bound for the case where the set $\left\{\left(f\left(x_1\right), \dots, f\left(x_N\right)\right) : f \in \cF\right\}$ is finite.

\begin{lemma}\label{lem:Massart}[\citet{Massart00:Some}]
Let $\cF|_{\sample} = \left\{\left(f\left(x_1\right), \dots, f\left(x_N\right)\right) : f \in \cF\right\}$. If $\left|\cF|_{\sample}\right|$ is finite, then $\erad(\cF) \leq \frac{r\sqrt{2 \log \left|\cF|_{\sample}\right|}}{N}$, where $r = \max_{\vec{a} \in \cF|_{\sample}} ||\vec{a}||_2$.
\end{lemma}

\section{Additional proofs and lemmas about Algorithm~\ref{alg:findSubset}}\label{app:algorithm}
The following Rademacher complexity bound guarantees that with probability $1-\zeta$, Algorithm~\ref{alg:findSubset} has a $\zeta$-representative run.

\begin{lemma}\label{lem:rad}
For any $\tau \in \Z_{\geq 0}$, define the function classes $\cF_{\tau} = \left\{j \mapsto \min\{\ell(\vec{\rho}, j), \tau\} \mid \vec{\rho} \in \cP\right\}$ and $\cH_{\tau} = \left\{j \mapsto \textbf{1}_{\{\ell(\vec{\rho}, j) \leq \tau\}} \mid \vec{\rho} \in \cP\right\}$. Let $\sample \subseteq \Pi$ be a set of problem instances. Then $\erad\left(\cF_{\tau}\right) \leq \tau\sqrt{\frac{2d \ln \left|\gp(\sample, \tau)\right|}{|\sample|}}$ and $\erad\left(\cH_{\tau}\right) \leq \sqrt{\frac{2d \ln \left|\gp(\sample, \tau)\right|}{|\sample|}}$.
\end{lemma}

\begin{proof}
This lemma follows from Lemma~\ref{lem:Massart}.
\end{proof}

\begin{lemma}\label{lem:deltaSandwich}
Suppose Algorithm~\ref{alg:findSubset} has a $\zeta$-representative run. For any set $\cP' \in \bar{\cG}$ and all parameters $\vec{\rho} \in \cP'$, $t_{\delta/2}(\vec{\rho}) \leq \tau_{\cP'} \leq t_{\delta/4}(\vec{\rho})$.
\end{lemma}

\begin{proof}
We begin by proving that $t_{\delta/2}(\vec{\rho}) \leq \tau_{\cP'}$.

\begin{claim}\label{claim:tau_lb}
Suppose Algorithm~\ref{alg:findSubset} has a $\zeta$-representative run. For any set $\cP' \in \cG$ and all parameters $\vec{\rho} \in \cP'$, $t_{\delta/2}(\vec{\rho}) \leq \tau_{\cP'}$.
\end{claim}

\begin{proof}[Proof of Claim~\ref{claim:tau_lb}]
Let $t$ be the round that $\cP'$ is added to $\cG$. We know that for all parameter vectors $\vec{\rho} \in \cP'$, \begin{equation}\frac{1}{\left|\sample_t\right|} \sum_{j \in \sample_t} \textbf{1}_{\left\{\ell\left(\vec{\rho}, j\right) \leq \tau_{\cP'}\right\}} \geq \frac{\lfloor \left|\sample_t\right|(1 - 3\delta/8)\rfloor}{\left|\sample_t\right|} \geq 1 - \frac{3\delta}{8}- \frac{1}{\left|\sample_t\right|}.\label{eq:lb_tau}\end{equation} Since Algorithm~\ref{alg:findSubset} had a $\zeta$-representative run, we know that \begin{align*}
\Pr_{j \sim \Gamma}\left[\ell(\vec{\rho}, j) \leq \tau_{\cP'}\right] &\geq \frac{1}{\left|\sample_t\right|} \sum_{j \in \sample_t} \textbf{1}_{\left\{\ell\left(\vec{\rho}, j\right) \leq \tau_{\cP'}\right\}} - \sqrt{\frac{2d \ln \left|\gp\left(\sample_t, \tau_{\cP'}\right)\right|}{\left|\sample_t\right|}} - 2\sqrt{\frac{2}{\left|\sample_t\right|}\ln \frac{8\left(\tau_{\cP'} \left|\sample_t\right|t\right)^2}{\zeta}}\\
&\geq \frac{1}{\left|\sample_t\right|} \sum_{j \in \sample_t} \textbf{1}_{\left\{\ell\left(\vec{\rho}, j\right) \leq \tau_{\cP'}\right\}} - \sqrt{\frac{2d \ln \left|\gp\left(\sample_t, 2^t\right)\right|}{\left|\sample_t\right|}} - 2\sqrt{\frac{2}{\left|\sample_t\right|}\ln \frac{8\left(2^t \left|\sample_t\right|t\right)^2}{\zeta}},
\end{align*} where the second inequality follows from the fact that $\tau_{\cP'} \leq 2^t$ and monotonicity. Based on Step~\ref{step:growS} of Algorithm~\ref{alg:findSubset}, we know that \[\Pr_{j \sim \Gamma}\left[\ell(\vec{\rho}, j) \leq \tau_{\cP'}\right] \geq \frac{1}{\left|\sample_t\right|} \sum_{j \in \sample_t} \textbf{1}_{\left\{\ell\left(\vec{\rho}, j\right) \leq \tau_{\cP'}\right\}} - \eta\delta.\] Moreover, by Equation~\eqref{eq:lb_tau}, $\Pr_{j \sim \Gamma}\left[\ell(\vec{\rho}, j) \leq \tau_{\cP'}\right] \geq 1 - \frac{3\delta}{8}- \frac{1}{\left|\sample_t\right|} - \eta\delta.$ Based on Step~\ref{step:growS} of Algorithm~\ref{alg:findSubset}, we also know that $\frac{1}{\sqrt{\left|\sample_t\right|}} \leq \eta\delta$. Therefore, $\Pr_{j \sim \Gamma}\left[\ell(\vec{\rho}, j) \leq \tau_{\cP'}\right] \geq 1 - \left(\frac{3}{8} + \eta^2 + \eta\right)\delta.$ Finally, since $\eta \leq \frac{1}{9}$, we have that $\Pr_{j \sim \Gamma}\left[\ell(\vec{\rho}, j) \leq \tau_{\cP'}\right] > 1 - \delta/2$, which means that \begin{equation}\Pr_{j \sim \Gamma}\left[\ell(\vec{\rho}, j) > \tau_{\cP'}\right] = \Pr_{j \sim \Gamma}\left[\ell(\vec{\rho}, j) \geq \tau_{\cP'} + 1\right] < \frac{\delta}{2}.\label{eq:tauPlusOne}\end{equation}

We claim that Equation~\eqref{eq:tauPlusOne} implies that $t_{\delta/2}(\vec{\rho}) \leq \tau_{\cP'}$. For a contradiction, suppose $t_{\delta/2}(\vec{\rho}) > \tau_{\cP'}$, or in other words, $t_{\delta/2}(\vec{\rho}) \geq \tau_{\cP'} + 1$.
Since $t_{\delta/2}(\vec{\rho}) = \argmax_{\tau \in \Z}\left\{\Pr_{j \sim \Gamma}[\ell(\vec{\rho}, j) \geq \tau] \geq \delta/2\right\}$, this would mean that $\delta/2 \leq \Pr_{j \sim \Gamma}\left[\ell(\vec{\rho}, j) \geq t_{\delta/2}(\vec{\rho})\right] \leq \Pr_{j \sim \Gamma}\left[\ell(\vec{\rho}, j) \geq \tau_{\cP'} + 1\right] < \delta/2$, which is a contradiction. Therefore, the claim holds.
\end{proof}

Next, we prove that $\tau_{\cP'} \leq t_{\delta/4}(\vec{\rho})$.

\begin{claim}\label{claim:tau_ub}
Suppose Algorithm~\ref{alg:findSubset} has a $\zeta$-representative run. For any set $\cP' \in \cG$ and all parameters $\vec{\rho} \in \cP'$, $\tau_{\cP'} \leq t_{\delta/4}(\vec{\rho})$.
\end{claim}

\begin{proof}[Proof of Claim~\ref{claim:tau_ub}]
Let $t$ be the round that $\cP'$ is added to $\cG$. We know that for all parameter vectors $\vec{\rho} \in \cP'$, \begin{align*}1 - \frac{1}{\left|\sample_t\right|} \sum_{j \in \sample_t} \textbf{1}_{\left\{\ell\left(\vec{\rho}, j\right) \leq \tau_{\cP'} - 1\right\}}
&= 1 - \frac{1}{\left|\sample_t\right|} \sum_{j \in \sample_t} \textbf{1}_{\left\{\ell\left(\vec{\rho}, j\right) < \tau_{\cP'}\right\}}\\
&= \frac{1}{\left|\sample_t\right|} \sum_{j \in \sample_t} \left(1 - \textbf{1}_{\left\{\ell\left(\vec{\rho}, j\right) < \tau_{\cP'}\right\}}\right)\\
 &= \frac{1}{\left|\sample_t\right|} \sum_{j \in \sample_t} \textbf{1}_{\left\{\ell\left(\vec{\rho}, j\right) \geq \tau_{\cP'}\right\}}\\
&\geq \frac{\left|\sample_t\right| - \left\lfloor \left|\sample_t\right|(1 - 3\delta/8)\right\rfloor}{\left|\sample_t\right|}\\
&\geq \frac{3\delta}{8}.\end{align*} Therefore, $\frac{1}{\left|\sample_t\right|} \sum_{j \in \sample_t} \textbf{1}_{\left\{\ell\left(\vec{\rho}, j\right) \leq \tau_{\cP'} - 1\right\}} \leq 1 - \frac{3\delta}{8}$. Since Algorithm~\ref{alg:findSubset} had a $\zeta$-representative run, we know that \begin{align*}
&\Pr_{j \sim \Gamma}\left[\ell(\vec{\rho}, j) \leq \tau_{\cP'} - 1\right]\\
\leq\text{ }&\frac{1}{\left|\sample_t\right|} \sum_{j \in \sample_t} \textbf{1}_{\left\{\ell\left(\vec{\rho}, j\right) \leq \tau_{\cP'} - 1\right\}} + \sqrt{\frac{2d \ln \left|\gp\left(\sample_t, \tau_{\cP'} - 1\right)\right|}{\left|\sample_t\right|}} + 2\sqrt{\frac{2}{\left|\sample_t\right|}\ln \frac{8\left(\left(\tau_{\cP'} - 1\right) \left|\sample_t\right| t\right)^2}{\zeta}}\\
\leq\text{ }& 1 - \frac{3\delta}{8} + \sqrt{\frac{2d \ln \left|\gp\left(\sample_t, \tau_{\cP'} - 1\right)\right|}{\left|\sample_t\right|}} + 2\sqrt{\frac{2}{\left|\sample_t\right|}\ln \frac{8\left(\left(\tau_{\cP'} - 1\right) \left|\sample_t\right| t\right)^2}{\zeta}}\\
<\text{ }& 1 - \frac{3\delta}{8} + \sqrt{\frac{2d \ln \left|\gp\left(\sample_t, 2^t\right)\right|}{\left|\sample_t\right|}} + 2\sqrt{\frac{2}{\left|\sample_t\right|}\ln \frac{8\left(2^t \left|\sample_t\right| t\right)^2}{\zeta}}
\end{align*} because $\tau_{\cP'} - 1 < \tau_{\cP'} \leq 2^t$ and monotonicity. Based on Step~\ref{step:growS} of Algorithm~\ref{alg:findSubset}, \[\Pr_{j \sim \Gamma}\left[\ell(\vec{\rho}, j) \leq \tau_{\cP'} - 1\right] < 1 - \frac{3\delta}{8} + \eta\delta.\] Since $\eta \leq 1/9$, we have that $\Pr_{j \sim \Gamma}\left[\ell(\vec{\rho}, j) \leq \tau_{\cP'} - 1\right] < 1 - \delta/4.$ Therefore, $\Pr_{j \sim \Gamma}\left[\ell(\vec{\rho}, j) \geq \tau_{\cP'}\right] = \Pr_{j \sim \Gamma}\left[\ell(\vec{\rho}, j) > \tau_{\cP'} - 1\right] > \delta/4.$ Since \[t_{\delta/4}(\vec{\rho}) = \argmax_{\tau \in \Z}\left\{\Pr_{j \sim \Gamma}[\ell(\vec{\rho}, j) \geq \tau] \geq \delta/4\right\},\] we have that $\tau_{\cP'} \leq t_{\delta/4}(\vec{\rho})$.
\end{proof}
The lemma statement follows from Claims~\ref{claim:tau_lb} and \ref{claim:tau_ub}.
\end{proof}

\begin{cor}\label{cor:exp_LB}
Suppose Algorithm~\ref{alg:findSubset} has a $\zeta$-representative run. For every set $\cP' \in \bar{\cG}$ and any parameter vector $\vec{\rho} \in \cP'$, $\tau_{\cP'}\delta/4 \leq \E_{j \sim \Gamma}\left[\min \left\{\ell(\vec{\rho}, j), \tau_{\cP'}\right\}\right].$
\end{cor}

\begin{proof}
By Lemma~\ref{lem:deltaSandwich}, we know that $\tau_{\cP'} \leq t_{\delta/4}(\vec{\rho})$, so $\Pr_{j \sim \Gamma}\left[\ell(\vec{\rho}, j) \geq \tau_{\cP'}\right] \geq \frac{\delta}{4}$. Therefore, $\E_{j \sim \Gamma}\left[\min \left\{\ell(\vec{\rho}, j), \tau_{\cP'}\right\}\right] \geq \tau_{\cP'} \Pr_{j \sim \Gamma}\left[\ell(\vec{\rho}, j) \geq \tau_{\cP'}\right] \geq\frac{\tau_{\cP'}\delta}{4}.$
\end{proof}

\BadParamLB*

\begin{proof}
The last round that Algorithm~\ref{alg:findSubset} adds any subset to the set $\cG$ is round $\bar{t}-1$. For ease of notation, let $\bar{\sample} = \sample_{\bar{t}-1}$. Since $\vec{\rho}$ is not an element of any set in $\cG$, the cap $2^{\bar{t}-1}$ must be too small compared to the average loss of the parameter $\vec{\rho}$. Specifically, it must be that $\frac{1}{\left|\bar{\sample}\right|} \sum_{j \in \bar{\sample}} \textbf{1}_{\left\{\ell\left(\vec{\rho}, j\right) \leq 2^{\bar{t}-1}\right\}} < 1 - \frac{3\delta}{8}$. Otherwise, the algorithm would have added a parameter set containing $\vec{\rho}$ to the set $\cG$ on round $\bar{t}-1$ (Step~\ref{step:delta}).
Since Algorithm~\ref{alg:findSubset} had a $\zeta$-representative run, we know that
the probability the loss of $\vec{\rho}$ is smaller than $2^{\bar{t}-1}$ converges to the fraction of samples with loss smaller than $2^{\bar{t}-1}$. Specifically,
\begin{align*}&\Pr_{j \sim \Gamma}\left[\ell(\vec{\rho}, j)\leq 2^{\bar{t}-1}\right]\\
\leq \text{ }&\frac{1}{\left|\bar{\sample}\right|} \sum_{j \in \bar{\sample}} \textbf{1}_{\left\{\ell\left(\vec{\rho}, j\right) \leq 2^{\bar{t}-1}\right\}} + \sqrt{\frac{2d \ln \left|\gp\left(\bar{\sample}, 2^{\bar{t}-1}\right)\right|}{\left|\bar{\sample}\right|}} + 2\sqrt{\frac{2}{\left|\bar{\sample}\right|}\ln \frac{8\left(2^{\bar{t}-1}\left|\bar{\sample}\right|\left(\bar{t}-1\right)\right)^2}{\zeta}}\\
\leq\text{ } &\frac{1}{\left|\bar{\sample}\right|} \sum_{j \in \bar{\sample}} \textbf{1}_{\left\{\ell\left(\vec{\rho}, j\right) \leq 2^{\bar{t}-1}\right\}} + \eta\delta,\end{align*} where the second inequality follows from Step~\ref{step:growS} of Algorithm~\ref{alg:findSubset}.
Using our bound of $1 - 3\delta/8$ on the fraction of samples with loss smaller than $2^{\bar{t}-1}$, we have that
$\Pr_{j \sim \Gamma}\left[\ell(\vec{\rho}, j) \leq 2^{\bar{t}-1}\right] < 1 - (3/8 - \eta)\delta$. Since $\eta \leq 1/9$, it must be that $\Pr_{j \sim \Gamma}\left[\ell(\vec{\rho}, j) \leq 2^{\bar{t}-1}\right] < 1 - \delta/4$, or conversely, $\Pr_{j \sim \Gamma}\left[\ell(\vec{\rho}, j) \geq 2^{\bar{t}-1}\right] \geq \Pr_{j \sim \Gamma}\left[\ell(\vec{\rho}, j) > 2^{\bar{t}-1}\right] > \delta/4$. Since \[t_{\delta/4}(\vec{\rho}) = \argmax_{\tau \in \Z}\left\{ \Pr_{j \sim \Gamma}[\ell(\vec{\rho}, j) \geq \tau] \geq \delta/4\right\},\] we have that $2^{\bar{t}-1} \leq t_{\delta/4}(\vec{\rho})$. Therefore, \begin{align*}\E_{j \sim \Gamma}\left[\min\left\{\ell(\vec{\rho}, j), t_{\delta/4}(\vec{\rho})\right\}\right] &\geq t_{\delta/4}(\vec{\rho})\Pr_{j \sim \Gamma}\left[\ell(\vec{\rho}, j)\geq t_{\delta/4}(\vec{\rho})\right]\\
&\geq \frac{\delta}{4} \cdot t_{\delta/4}(\vec{\rho})\\
&\geq 2^{{\bar{t}}-3}\delta.\end{align*}
\end{proof}

\begin{lemma}\label{lem:Tub}
Suppose Algorithm~\ref{alg:findSubset} has a $\zeta$-representative run. For any parameter vector $\vec{\rho} \not\in \bigcup_{\cP' \in \bar{\cG}} \cP'$, $\E_{j \sim \Gamma}\left[\min\left\{\ell(\vec{\rho}, j), t_{\delta/4}(\vec{\rho})\right\}\right] \geq \bar{T}.$
\end{lemma}

\begin{proof}
From Lemma~\ref{lem:bad_param_lb}, we know that $\E_{j \sim \Gamma}\left[\min\left\{\ell(\vec{\rho}, j), t_{\delta/4}(\vec{\rho})\right\}\right] \geq 2^{\bar{t}-3}\delta.$
Moreover, from Step~\ref{step:outer_while} of Algorithm~\ref{alg:findSubset}, $\bar{T} \leq 2^{\bar{t}-3}\delta$, so $\E_{j \sim \Gamma}\left[\min\left\{\ell(\vec{\rho}, j), t_{\delta/4}(\vec{\rho})\right\}\right] \geq \bar{T}.$
\end{proof}

\begin{lemma}\label{lem:Tlb}
Suppose Algorithm~\ref{alg:findSubset} has a $\zeta$-representative run. There exists a set $\cP' \in \bar{\cG}$ and a parameter vector $\vec{\rho} \in \cP'$ such that $\bar{T} \geq \frac{1}{\sqrt[4]{1+\epsilon}} \cdot \E_{j \sim \Gamma}\left[\min \left\{\ell(\vec{\rho}, j), \tau_{\cP'}\right\}\right]$.
\end{lemma}

\begin{proof}
By definition of the upper confidence bound $\bar{T}$ (Steps~\ref{step:sort} through \ref{step:updateT} of Algorithm~\ref{alg:findSubset}), there is some round $t$, some set $\cP' \in \bar{\cG}$, and some parameter vector $\vec{\rho} \in \cP'$ such that $\bar{T} = \frac{1}{\left|\sample_t\right|}\sum_{j \in \sample_t} \min\left\{\ell\left(\vec{\rho}, j\right), \tau_{\cP'}\right\}$. Since Algorithm~\ref{alg:findSubset} had a $\zeta$-representative run, \begin{align*}\bar{T} &= \frac{1}{\left|\sample_t\right|}\sum_{j \in \sample_t} \min\left\{\ell\left(\vec{\rho}, j\right), \tau_{\cP'}\right\}\\
&\geq \E_{j \sim \Gamma}\left[\min\left\{\ell\left(\vec{\rho}, j\right), \tau_{\cP'}\right\}\right] - \tau_{\cP'}\left( \sqrt{\frac{2d \ln \left|\gp\left(\sample_t, \tau_{\cP'}\right)\right|}{\left|\sample_t\right|}} + 2\sqrt{\frac{2}{\left|\sample_t\right|}\ln \frac{8\left(\tau_{\cP'} \left|\sample_t\right| t\right)^2}{\zeta}}\right).\end{align*} By Step~\ref{step:delta}, we know that at least a $(1-3\delta/8)$-fraction of the problem instances $j \in \sample_t$ have a loss $\ell(\vec{\rho}, j)$ that is at most $2^t$. Therefore, by definition of $\tau_{\cP'} = \tau_{\lfloor |\sample_t|\left(1 - 3\delta/8\right)\rfloor}$, it must be that $\tau_{\cP'} \leq 2^t$. By monotonicity, this means that \[\bar{T} \geq \E_{j \sim \Gamma}\left[\min\left\{\ell\left(\vec{\rho}, j\right), \tau_{\cP'}\right\}\right] - \tau_{\cP'}\left( \sqrt{\frac{2d \ln \left|\gp\left(\sample_t, 2^t\right)\right|}{\left|\sample_t\right|}} + 2\sqrt{\frac{2}{\left|\sample_t\right|}\ln \frac{8\left(2^t \left|\sample_t\right| t\right)^2}{\zeta}}\right).\] Based on Step~\ref{step:growS} of Algorithm~\ref{alg:findSubset}, \[\bar{T} \geq \E_{j \sim \Gamma}\left[\min\left\{\ell\left(\vec{\rho}, j\right), \tau_{\cP'}\right\}\right] - \tau_{\cP'}\eta\delta.\] From Corollary~\ref{cor:exp_LB}, $\tau_{\cP'}\delta/4 \leq \E_{j \sim \Gamma}\left[\min \left\{\ell(\vec{\rho}, j), \tau_{\cP'}\right\}\right]$, which means that \[\bar{T} \geq (1 - 4\eta)\E_{j \sim \Gamma}\left[\min \left\{\ell(\vec{\rho}, j), \tau_{\cP'}\right\}\right].\] Finally, the lemma statement follows from the fact that \[\eta = \min\left\{\frac{1}{8}\left(\sqrt[4]{1 + \epsilon} - 1\right), \frac{1}{9}\right\} \leq \frac{1}{4}\left(1 - \frac{1}{\sqrt[4]{1 + \epsilon}}\right).\]
\end{proof}

\begin{lemma}\label{lem:similar}
Suppose Algorithm~\ref{alg:findSubset} has a $\zeta$-representative run. For every set $\cP' \in \bar{\cG}$ and every pair of parameter vectors $\vec{\rho}_1, \vec{\rho}_2 \in \cP'$, $\E_{j \sim \Gamma}\left[\min\left\{\ell\left(\vec{\rho}_1, j\right), \tau_{\cP'}\right\}\right] \leq \sqrt[4]{1 + \epsilon} \cdot \E_{j \sim \Gamma}\left[\min \left\{\ell\left(\vec{\rho}_2, j\right), \tau_{\cP'}\right\}\right]$.
\end{lemma}

\begin{proof}
Let $t$ be the round that the interval $\cP'$ was added to $\cG$. Since Algorithm~\ref{alg:findSubset} had a $\zeta$-representative run, \begin{align*}&\E_{j \sim \Gamma}\left[\min\left\{\ell\left(\vec{\rho}_1, j\right), \tau_{\cP'}\right\}\right]\\
\leq\text{ } &\frac{1}{\left|\sample_t\right|} \sum_{j \in \sample_t}\min\left\{\ell\left(\vec{\rho}_1, j\right), \tau_{\cP'}\right\} + \tau_{\cP'}\left( \sqrt{\frac{2d \ln \left|\gp\left(\sample_t, \tau_{\cP'}\right)\right|}{\left|\sample_t\right|}} + 2\sqrt{\frac{2}{\left|\sample_t\right|}\ln \frac{8\left(\tau_{\cP'} \left|\sample_t\right| t\right)^2}{\zeta}}\right).\end{align*}
By definition of the set $\cP'$, for all problem instances $j \in \sample_t$, $\min\left\{\ell\left(\vec{\rho}_1, j\right), \tau_{\cP'}\right\} = \min\left\{\ell\left(\vec{\rho}_2, j\right), \tau_{\cP'}\right\}.$ Therefore, \begin{align*}&\E_{j \sim \Gamma}\left[\min\left\{\ell\left(\vec{\rho}_1, j\right), \tau_{\cP'}\right\}\right]\\
\leq \text{ } &\frac{1}{\left|\sample_t\right|} \sum_{j \in \sample_t}\min\left\{\ell\left(\vec{\rho}_2, j\right), \tau_{\cP'}\right\} + \tau_{\cP'}\left( \sqrt{\frac{2d \ln \left|\gp\left(\sample_t, \tau_{\cP'}\right)\right|}{\left|\sample_t\right|}} + 2\sqrt{\frac{2}{\left|\sample_t\right|}\ln \frac{8\left(\tau_{\cP'} \left|\sample_t\right| t\right)^2}{\zeta}}\right).\end{align*} Again, since Algorithm~\ref{alg:findSubset} had a $\zeta$-representative run, \begin{align*}&\E_{j \sim \Gamma}\left[\min\left\{\ell\left(\vec{\rho}_1, j\right), \tau_{\cP'}\right\}\right]\\
\leq \text{ } &\E_{j \sim \Gamma}\left[\min\left\{\ell\left(\vec{\rho}_2, j\right), \tau_{\cP'}\right\}\right] + 2\tau_{\cP'}\left( \sqrt{\frac{2d \ln \left|\gp\left(\sample_t, \tau_{\cP'}\right)\right|}{\left|\sample_t\right|}} + 2\sqrt{\frac{2}{\left|\sample_t\right|}\ln \frac{8\left(\tau_{\cP'} \left|\sample_t\right| t\right)^2}{\zeta}}\right).\end{align*}
 Since $\tau_{\cP'} \leq 2^t$,
 \begin{align*}&\E_{j \sim \Gamma}\left[\min\left\{\ell\left(\vec{\rho}_1, j\right), \tau_{\cP'}\right\}\right]\\
\leq \text{ } &\E_{j \sim \Gamma}\left[\min\left\{\ell\left(\vec{\rho}_2, j\right), \tau_{\cP'}\right\}\right] + 2\tau_{\cP'}\left( \sqrt{\frac{2d \ln \left|\gp\left(\sample_t, 2^t\right)\right|}{\left|\sample_t\right|}} + 2\sqrt{\frac{2}{\left|\sample_t\right|}\ln \frac{8\left(2^t \left|\sample_t\right| t\right)^2}{\zeta}}\right).
\end{align*}
Based on Step~\ref{step:growS} of Algorithm~\ref{alg:findSubset}, this means that
\[\E_{j \sim \Gamma}\left[\min\left\{\ell\left(\vec{\rho}_1, j\right), \tau_{\cP'}\right\}\right] \leq \E_{j \sim \Gamma}\left[\min\left\{\ell\left(\vec{\rho}_2, j\right), \tau_{\cP'}\right\}\right] + 2\tau_{\cP'}\eta\delta.\]
By Corollary~\ref{cor:exp_LB}, $\E_{j \sim \Gamma}\left[\min\left\{\ell\left(\vec{\rho}_1, j\right), \tau_{\cP'}\right\}\right] \leq (1+8\eta)\E_{j \sim \Gamma}\left[\min \left\{\ell\left(\vec{\rho}_2, j\right), \tau_{\cP'}\right\}\right]$. The lemma statement follows from the fact that $\eta \leq \left(\sqrt[4]{1 + \epsilon} - 1 \right)/8$.
\end{proof}

\optSubset*

\begin{proof}
Since $OPT_{\delta/4} := \inf_{\vec{\rho} \in \cP} \left\{\E_{j \sim \Gamma}\left[\min\left\{\ell\left(\vec{\rho}, j\right), t_{\delta/4}(\vec{\rho})\right\}\right]\right\}$, we know that for any $\gamma > 0$, there exists a parameter vector $\vec{\rho} \in \cP$ such that $\E_{j \sim \Gamma}\left[\min\left\{\ell(\vec{\rho}, j), t_{\delta/4}(\vec{\rho})\right\}\right] \leq OPT_{\delta/4} + \gamma.$ We claim there exists a parameter $\vec{\rho}' \in \cP^*$ such that \begin{equation}\E_{j \sim \Gamma}\left[\min\left\{\ell\left(\vec{\rho}', j\right), t_{\delta/2}\left(\vec{\rho}'\right)\right\}\right] \leq \sqrt{1 + \epsilon} \cdot \E_{j \sim \Gamma}\left[\min\left\{\ell(\vec{\rho}, j), t_{\delta/4}(\vec{\rho})\right\}\right],\label{eq:inSet2}\end{equation} and thus the lemma statement holds (see Lemma~\ref{lem:opt_subset_inf}).

First, suppose $\vec{\rho}$ is contained in a set $\cP' \in \bar{\cG}$. By Lemmas~\ref{lem:deltaSandwich} and \ref{lem:similar}, there exists a parameter $\vec{\rho}' \in \cP' \cap \cP^*$ such that \begin{align*}\E_{j \sim \Gamma}\left[\min\left\{\ell\left(\vec{\rho}', j\right), t_{\delta/2}\left(\vec{\rho}'\right)\right\}\right] &\leq \E_{j \sim \Gamma}\left[\min\left\{\ell\left(\vec{\rho}', j\right), \tau_{\cP'}\right\}\right]\\
&\leq \sqrt[4]{1+\epsilon} \cdot \E_{j \sim \Gamma}\left[\min\left\{\ell\left(\vec{\rho}, j\right), \tau_{\cP'}\right\}\right]\\
&\leq \sqrt[4]{1+\epsilon} \cdot \E_{j \sim \Gamma}\left[\min\left\{\ell\left(\vec{\rho}, j\right), t_{\delta/4}(\vec{\rho})\right\}\right].\end{align*} Since $\sqrt[4]{1+\epsilon} \leq \sqrt{1+\epsilon}$, Equation~\eqref{eq:inSet2} holds in this case.

Otherwise, suppose $\vec{\rho} \not\in \bigcup_{\cP' \in \bar{\cG}} \cP'$. By Lemma~\ref{lem:Tub}, we know that $\E_{j \sim \Gamma}\left[\min\left\{\ell(\vec{\rho}, j), t_{\delta/4}(\vec{\rho})\right\}\right] \geq \bar{T}.$ Moreover, by Lemma~\ref{lem:Tlb}, there exists a set $\cP' \in \bar{\cG}$ and parameter vector $\vec{\rho}^* \in \cP'$ such that $\sqrt[4]{1+\epsilon} \cdot \bar{T} \geq \E_{j \sim \Gamma}\left[\min \left\{\ell\left(\vec{\rho}^*, j\right), \tau_{\cP'}\right\}\right]$. Finally, by Lemma~\ref{lem:similar}, there exists a parameter vector $\vec{\rho}' \in \cP' \cap \cP^*$ such that \begin{align*}\E_{j \sim \Gamma}\left[\min\left\{\ell\left(\vec{\rho}', j\right), t_{\delta/2}\left(\vec{\rho}'\right)\right\}\right] &\leq \E_{j \sim \Gamma}\left[\min\left\{\ell\left(\vec{\rho}', j\right), \tau_{\cP'}\right\}\right]\\
&\leq \sqrt[4]{1+\epsilon} \cdot \E_{j \sim \Gamma}\left[\min\left\{\ell\left(\vec{\rho}^*, j\right), \tau_{\cP'}\right\}\right]\\
&\leq \sqrt{1+\epsilon} \cdot \bar{T}\\
&< \sqrt{1+\epsilon} \cdot \E_{j \sim \Gamma}\left[\min\left\{\ell(\vec{\rho}, j), t_{\delta/4}(\vec{\rho})\right\}\right].\end{align*} Therefore, Equation~\eqref{eq:inSet2} holds in this case as well.
\end{proof}

\begin{lemma}\label{lem:opt_subset_inf}
Let $\cP^*$ be the set of parameters output by Algorithm~\ref{alg:findSubset}. Suppose that for every $\gamma > 0$, there exists a parameter vector $\vec{\rho}' \in \cP^*$ such that $\E_{j \sim \Gamma}\left[\min\left\{\ell(\vec{\rho}', j), t_{\delta/2}(\vec{\rho}')\right\}\right] \leq \sqrt{1 + \epsilon}\left(OPT_{\delta/4} + \gamma\right).$ Then $\min_{\vec{\rho} \in \cP^*}\left\{\E_{j \sim \Gamma}\left[\min\left\{\ell(\vec{\rho}, j), t_{\delta/2}(\vec{\rho})\right\}\right]\right\} \leq \sqrt{1 + \epsilon}\cdot OPT_{\delta/4}.$
\end{lemma}

\begin{proof}
For a contradiction, suppose that $\min_{\vec{\rho} \in \cP^*}\left\{\E_{j \sim \Gamma}\left[\min\left\{\ell(\vec{\rho}, j), t_{\delta/2}(\vec{\rho})\right\}\right]\right\} > \sqrt{1 + \epsilon}\cdot OPT_{\delta/4}$ and let $\gamma' = \min_{\vec{\rho} \in \cP^*}\left\{\E_{j \sim \Gamma}\left[\min\left\{\ell(\vec{\rho}, j), t_{\delta/2}(\vec{\rho})\right\}\right]\right\} - \sqrt{1 + \epsilon}\cdot OPT_{\delta/4}.$ Setting $\gamma = \frac{\gamma'}{2\sqrt{1 +\epsilon}}$, we know there exists a parameter vector $\vec{\rho}' \in \cP^*$ such that \begin{align*}\E_{j \sim \Gamma}\left[\min\left\{\ell(\vec{\rho}', j), t_{\delta/2}(\vec{\rho}')\right\}\right] &\leq \sqrt{1 + \epsilon}\left(OPT_{\delta/4} + \frac{\gamma'}{2\sqrt{1 +\epsilon}}\right)\\
&= \sqrt{1 + \epsilon} \cdot OPT_{\delta/4} + \frac{\gamma'}{2}\\
&= \min_{\vec{\rho} \in \cP^*}\left\{\E_{j \sim \Gamma}\left[\min\left\{\ell(\vec{\rho}, j), t_{\delta/2}(\vec{\rho})\right\}\right]\right\} - \frac{\gamma'}{2}\\
&< \min_{\vec{\rho} \in \cP^*}\left\{\E_{j \sim \Gamma}\left[\min\left\{\ell(\vec{\rho}, j), t_{\delta/2}(\vec{\rho})\right\}\right]\right\},
\end{align*} which is a contradiction.
\end{proof}

\BoundBart*

\begin{proof}
For each set $\cP' \in \bar{\cG}$, let $t_{\cP'}$ be the round where $\cP'$ is added to the set $\cG$, let $\sample_{\cP'} = \sample_{t_{\cP'}}$, and let $\vec{\rho}_{\cP'}$ be an arbitrary parameter vector in $\cP'$. Since no set is added to $\cG$ when $t = \bar{t}$, it must be that for all sets $\cP' \in \bar{\cG}$, $t_{\cP'} \leq \bar{t}-1$. Moreover, since $OPT_{\delta/4} := \inf_{\vec{\rho} \in \cP} \left\{\E_{j \sim \Gamma}\left[\min\left\{\ell\left(\vec{\rho}, j\right), t_{\delta/4}(\vec{\rho})\right\}\right]\right\}$, we know that for every $\gamma > 0$, there exists a parameter vector $\vec{\rho}^*$ such that \[\E_{j \sim \Gamma}\left[\min\left\{\ell\left(\vec{\rho}^*, j\right), t_{\delta/4}\left(\vec{\rho}^*\right)\right\}\right] \leq OPT_{\delta/4} + \gamma.\] Below, we prove that $2^{\bar{t}} \leq \frac{16\sqrt[4]{1+\epsilon}}{\delta} \cdot \E_{j \sim \Gamma}\left[\min\left\{\ell\left(\vec{\rho}^*, j\right), t_{\delta/4}\left(\vec{\rho}^*\right)\right\}\right]$ and thus the lemma statement holds (see Lemma~\ref{lem:inf_bound_bar_t}).

\paragraph{Case 1: $\vec{\rho}^* \not\in \bigcup_{\cP' \in \bar{\cG}} \cP'$.} By Lemma~\ref{lem:bad_param_lb}, we know that $2^{\bar{t}-3} \delta \leq \E_{j \sim \Gamma}\left[\min\left\{\ell\left(\vec{\rho}^*, j\right), t_{\delta/4}\left(\vec{\rho}^*\right)\right\}\right]$. Therefore, $2^{\bar{t}} \leq \frac{8}{\delta} \cdot \E_{j \sim \Gamma}\left[\min\left\{\ell\left(\vec{\rho}^*, j\right), t_{\delta/4}\left(\vec{\rho}^*\right)\right\}\right]\leq \frac{16\sqrt[4]{1+\epsilon}}{\delta} \cdot \E_{j \sim \Gamma}\left[\min\left\{\ell\left(\vec{\rho}^*, j\right), t_{\delta/4}\left(\vec{\rho}^*\right)\right\}\right]$.

\paragraph{Case 2: $\vec{\rho}^*$ is an element of a set $\cP' \in \bar{\cG}$ and $t_{\cP'} \leq \bar{t} - 2$.}
Let $T'$ be the value of $T$ at the beginning of round $\bar{t} - 1$. Since the algorithm does not terminate on round $\bar{t}-1$, it must be that $2^{\bar{t}-4} \delta < T'$. By definition of $T'$, 
\[2^{\bar{t} -4}\delta < T' = \min_{\bar{\cP} : t_{\bar{\cP}} \leq \bar{t} - 2} \frac{1}{\left|\sample_{\bar{\cP}}\right|} \sum_{j \in \sample_{\bar{\cP}}} \min \left\{\ell\left(\vec{\rho}_{\bar{\cP}}, j\right), \tau_{\bar{\cP}}\right\} \leq \frac{1}{\left|\sample_{\cP'}\right|} \sum_{j \in \sample_{\cP'}} \min \left\{\ell\left(\vec{\rho}^*, j\right), \tau_{\cP'}\right\}.\]
Since Algorithm~\ref{alg:findSubset} had a $\zeta$-representative run, $2^{\bar{t} -4}\delta$  is upper-bounded by \[\E_{j \sim \Gamma}\left[ \min \left\{\ell\left(\vec{\rho}^*, j\right), \tau_{\cP'}\right\}\right] + \tau_{\cP'}\left( \sqrt{\frac{2d \ln \left|\gp\left(\sample_{\cP'}, \tau_{\cP'}\right)\right|}{\left|\sample_{\cP'}\right|}} + 2\sqrt{\frac{2}{\left|\sample_{\cP'}\right|}\ln \frac{8\left(\tau_{\cP'} \left|\sample_{\cP'}\right|t_{\cP'}\right)^2}{\zeta}}\right).\] Since $\tau_{\cP'} \leq 2^{t_{\cP'}}$ and $f$ is monotone, $2^{\bar{t} -4}\delta$  is at most \[E_{j \sim \Gamma}\left[ \min \left\{\ell\left(\vec{\rho}^*, j\right), \tau_{\cP'}\right\}\right] + \tau_{\cP'}\left( \sqrt{\frac{2d \ln \left|\gp\left(\sample_{\cP'}, 2^{t_{\cP'}}\right)\right|}{\left|\sample_{\cP'}\right|}} + 2\sqrt{\frac{2}{\left|\sample_{\cP'}\right|}\ln \frac{8\left(2^{t_{\cP'}} \left|\sample_{\cP'}\right|t_{\cP'}\right)^2}{\zeta}}\right).\]
By Step~\ref{step:growS} of Algorithm~\ref{alg:findSubset}, $2^{\bar{t} -4}\delta  \leq \E_{j \sim \Gamma}\left[ \min \left\{\ell\left(\vec{\rho}^*, j\right), \tau_{\cP'}\right\}\right] + \tau_{\cP'}\eta\delta.$ Finally, by Corollary~\ref{cor:exp_LB}, $2^{\bar{t} -4}\delta  \leq (1 + 4\eta)\E_{j \sim \Gamma}\left[ \min \left\{\ell\left(\vec{\rho}^*, j\right), \tau_{\cP'}\right\}\right].$ Since $\eta < \left(\sqrt[4]{1 + \epsilon} - 1 \right)/4$, \[2^{\bar{t} -4}\delta < \sqrt[4]{1 + \epsilon}\cdot \E_{j \sim \Gamma}\left[ \min \left\{\ell\left(\vec{\rho}^*, j\right), \tau_{\cP'}\right\}\right].\]
Recalling that $\tau_{\cP'} \leq t_{\delta/4}\left(\vec{\rho}^*\right)$ by Lemma~\ref{lem:deltaSandwich}, we conclude that \[2^{\bar{t}} \leq \frac{16\sqrt[4]{1 + \epsilon}}{\delta} \cdot \E_{j \sim \Gamma}\left[ \min \left\{\ell\left(\vec{\rho}^*, j\right), t_{\delta/4}\left(\vec{\rho}^*\right)\right\}\right].\]

\paragraph{Case 3: $\vec{\rho}^*$ is not an element of any set $\bar{\cP} \in \bar{\cG}$ with $t_{\bar{\cP}} \leq \bar{t}-2$, but
$\vec{\rho}^*$ is an element of a set $\cP' \in \bar{\cG}$ with $t_{\cP'} = \bar{t} - 1$.} Let $\sample' = \sample_{\bar{t}-2}$ and let $\bar{\cP}$ be the set containing $\vec{\rho}^*$ in Step~\ref{step:partition} on round $\bar{t}-2$. Since $\bar{\cP}$ was not added to $\cG$ on round $\bar{t}-2$, we know that fewer than a $\left(1 - \frac{3\delta}{8}\right)$-fraction of the instances in $\sample'$ have a loss of at most $2^{\bar{t} - 2}$ when run with any parameter vector $\vec{\rho} \in \bar{\cP}$ (including $\vec{\rho}^*$). In other words, \[\frac{1}{|\sample'|} \sum_{j \in \sample'} \textbf{1}_{\left\{\ell\left(\vec{\rho}^*, j\right) \leq 2^{\bar{t} - 2}\right\}} <  1 - \frac{3\delta}{8}.\] Since Algorithm~\ref{alg:findSubset} had a $\zeta$-representative run, \begin{align*}&\Pr_{j \sim \Gamma}\left[\ell\left(\vec{\rho}^*, j\right) \leq 2^{\bar{t} - 2}\right]\\
\leq\text{ } &\frac{1}{|\sample'|} \sum_{j \in \sample'} \textbf{1}_{\left\{\ell\left(\vec{\rho}^*, j\right) \leq 2^{\bar{t} - 2}\right\}} + \sqrt{\frac{2d \ln \left|\gp\left(\sample', 2^{\bar{t} - 2}\right)\right|}{\left|\sample'\right|}} + 2\sqrt{\frac{2}{\left|\sample'\right|}\ln \frac{8\left(2^{\bar{t} - 2} \left|\sample'\right|\left(\bar{t}-2\right)\right)^2}{\zeta}}\\
<\text{ } &1 - \frac{3\delta}{8} + \sqrt{\frac{2d \ln \left|\gp\left(\sample', 2^{\bar{t} - 2}\right)\right|}{\left|\sample'\right|}} + 2\sqrt{\frac{2}{\left|\sample'\right|}\ln \frac{8\left(2^{\bar{t} - 2} \left|\sample'\right|\left(\bar{t}-2\right)\right)^2}{\zeta}}.\end{align*} Based on Step~\ref{step:growS} of Algorithm~\ref{alg:findSubset}, $\Pr_{j \sim \Gamma}\left[\ell\left(\vec{\rho}^*, j\right) \leq 2^{\bar{t} - 2}\right] < 1 - (3/8 - \eta)\delta$. Since $\eta \leq 1/9$, $\Pr_{j \sim \Gamma}\left[\ell\left(\vec{\rho}^*, j\right) \leq 2^{\bar{t} - 2}\right] < 1 - \delta/4$. Therefore, $\Pr_{j \sim \Gamma}\left[\ell\left(\vec{\rho}^*, j\right) \geq 2^{\bar{t} - 2}\right] \geq \Pr_{j \sim \Gamma}\left[\ell\left(\vec{\rho}^*, j\right) > 2^{\bar{t} - 2}\right] \geq \delta/4$. Since $t_{\delta/4}(\vec{\rho}^*) = \argmax_{\tau \in \Z}\left\{\Pr_{j \sim \Gamma}[\ell(\vec{\rho}^*, j) \geq \tau] \geq \delta/4\right\}$, we have that $2^{\bar{t} - 2} \leq t_{\delta/4}\left(\vec{\rho}^*\right)$. Therefore, \begin{align*}\E_{j \sim \Gamma}\left[ \min \left\{\ell\left(\vec{\rho}^*, j\right), t_{\delta/4}\left(\vec{\rho}^*\right)\right\}\right] &\geq t_{\delta/4}\left(\vec{\rho}^*\right) \Pr\left[\ell\left(\vec{\rho}^*, j\right) \geq t_{\delta/4}\left(\vec{\rho}^*\right)\right]\\
&\geq 2^{\bar{t} - 2} \Pr\left[\ell\left(\vec{\rho}^*, j\right) \geq t_{\delta/4}\left(\vec{\rho}^*\right)\right]\\
&\geq \delta 2^{\bar{t} - 4},\end{align*} which means that $2^{\bar{t}} \leq \frac{16}{\delta} \cdot \E_{j \sim \Gamma}\left[ \min \left\{\ell\left(\vec{\rho}^*, j\right), t_{\delta/4}\left(\vec{\rho}^*\right)\right\}\right]$.
\end{proof}

\begin{lemma}\label{lem:inf_bound_bar_t}
Suppose that for all $\gamma > 0$, there exists a parameter vector $\vec{\rho}^* \in \cP$  such that $2^{\bar{t}} \leq \frac{16\sqrt[4]{1+\epsilon}}{\delta} \cdot \E_{j \sim \Gamma}\left[\min\left\{\ell\left(\vec{\rho}^*, j\right), t_{\delta/4}\left(\vec{\rho}^*\right)\right\}\right] \leq \frac{16\sqrt[4]{1+\epsilon}}{\delta} \left(OPT_{\delta/4} + \gamma\right)$. Then $2^{\bar{t}} \leq \frac{16\sqrt[4]{1+\epsilon}}{\delta} \cdot OPT_{\delta/4}$.
\end{lemma}

\begin{proof}
For a contradiction, suppose $2^{\bar{t}} > \frac{16\sqrt[4]{1+\epsilon}}{\delta} \cdot OPT_{\delta/4}$ and let $\gamma' = 2^{\bar{t}} - \frac{16\sqrt[4]{1+\epsilon}}{\delta} \cdot OPT_{\delta/4}$. Letting $\gamma = \frac{\gamma'\delta}{32\sqrt[4]{1+\epsilon}}$, we know there exists a parameter vector $\vec{\rho}^* \in \cP$ such that \begin{align*}2^{\bar{t}} &\leq \frac{16\sqrt[4]{1+\epsilon}}{\delta} \cdot \E_{j \sim \Gamma}\left[\min\left\{\ell\left(\vec{\rho}^*, j\right), t_{\delta/4}\left(\vec{\rho}^*\right)\right\}\right]\\
&\leq \frac{16\sqrt[4]{1+\epsilon}}{\delta} \left(OPT_{\delta/4} + \frac{\gamma'\delta}{32\sqrt[4]{1+\epsilon}}\right)\\
&= \frac{16\sqrt[4]{1+\epsilon}}{\delta} \cdot OPT_{\delta/4} + \frac{\gamma'}{2}\\
&= 2^{\bar{t}} - \frac{\gamma'}{2}\\
&< 2^{\bar{t}},
\end{align*} which is a contradiction. Therefore, the lemma statement holds.
\end{proof}

\begin{lemma}\label{lem:size_opt_subset}
Suppose Algorithm~\ref{alg:findSubset} has a $\zeta$-representative run. The size of the set $\cP^* \subset \cP$ that Algorithm~\ref{alg:findSubset} returns is bounded by $\sum_{t = 1}^{\bar{t}}\left|\gp\left(\sample_t, \frac{16}{\delta}\sqrt[4]{1+\epsilon} \cdot OPT_{\delta/4}\right)\right|$.
\end{lemma}

\begin{proof}
Based on Step~\ref{step:arbitrary} of Algorithm~\ref{alg:findSubset}, the size of $\cP^*$ equals the size of the set $\bar{\cG}$. Algorithm~\ref{alg:findSubset} only adds sets to $\cG$ on Step~\ref{step:updateG}, and on each round $t$, the number of sets it adds is bounded by $\left|\gp\left(\sample_t, 2^t\right)\right|$. By monotonicity, we know that $\left|\gp\left(\sample_t, 2^t\right)\right| \leq \left|\gp\left(\sample_t, 2^{\bar{t}}\right)\right| \leq \left|\gp\left(\sample_t, \frac{16}{\delta}\sqrt[4]{1+\epsilon} \cdot OPT_{\delta/4}\right)\right|$. Therefore, \[\left|\bar{\cG}\right| = \left|\cP^*\right| \leq \sum_{t = 1}^{\bar{t}}\left|\gp\left(\sample_t, \frac{16}{\delta}\sqrt[4]{1+\epsilon} \cdot OPT_{\delta/4}\right)\right|.\]
\end{proof}

\section{Additional information about related research}\label{app:related}
In this section, we begin by surveying additional related research. We then describe the bounds on $\left|\gp(\sample, \tau)\right|$ that prior research has provided in the contexts of integer programming and clustering. We also describe an integer programming algorithm configuration problem where the best-known uniform convergence bound is exponential in the number of variables~\citep{Balcan18:Learning}, whereas our algorithm only requires $\tilde O\left(\left(\delta\eta\right)^{-2}\right)$ samples.

\subsection{Survey of additional related research}\label{app:survey}
A related line of research~\citep{Gupta17:PAC,Cohen-Addad17:Online,Balcan18:Dispersion,Balcan19:Semi,Alabi19:Learning} studies algorithm configuration in online learning settings, where the learner encounters a sequence---perhaps adversarially selected---of problem instances over a series of timesteps. The learner's goal is to select an algorithm on each timestep so that the learner has strong cumulative performance across all timesteps (as quantified by \emph{regret}, typically). In contrast, this paper is focused on the batch learning setting, where the problem instances are not adversarially generated but come from a fixed distribution.

A number of papers have explored Bayesian optimization as a tool for parameter optimization. These algorithms have strong performance in practice~\citep{Bergstra11:Algorithms,Hutter11:Bayesian,Hutter11:Sequential,Snoek12:Practical}, but often do not come with theoretical guarantees, which is the focus of our paper. Those papers that do include provable guarantees typically require that the loss function is smooth, as quantified by its Lipschitz constant~\citep{Brochu10:Tutorial} or its RKHS norm~\citep{Berkenkamp19:No}, which is not the case in our setting.

\subsection{Integer programming}\label{app:IP}
In this section, to match prior research, we use the notation $Q$ to denote an integer program (rather than the notation $j$, as in the main body).

\citet{Balcan18:Learning} study \emph{mixed integer linear programs} (MILPs)
where the goal is to maximize an objective function $\vec{c}^\top \vec{x}$ subject to the constraints that $A\vec{x} \leq \vec{b}$ and that some of the components of $\vec{x}$ are contained in $\{0,1\}$.
Given a MILP $Q$, we use the notation $\breve{\vec{x}}_Q = \left(\breve{x}_{Q}[1], \dots \breve{x}_{Q}[n]\right)$ to denote an optimal solution to the MILP's LP relaxation. We denote the optimal objective value to the MILP's LP relaxation as $\breve{c}_Q$, which means that $\breve{c}_Q = \vec{c}^{\top} \breve{\vec{x}}_Q$.

The most popular algorithm for solving MILPs is called \emph{branch-and-bound (B\&B)}, which we now describe at a high level.
Let $Q'$ be a MILP we want to solve. B\&B builds a search tree $\tree$ with $Q'$ at the root. At each round, the algorithm uses a \emph{node selection policy} (such as \emph{depth- or best-first search}) to choose a leaf of $\tree$. This leaf node corresponds to a MILP we denote as $Q$. Using a \emph{variable selection policy}, the algorithm then chooses one of that MILP's variables. Specifically, let $Q^+_i$ (resp., $Q^-_i$) equal the MILP $Q$ after adding the constraint $x_i = 1$ (resp., $x_i = 0$). The algorithm defines the right (resp., left) child of the leaf $\node$ to equal $Q^+_i$ (resp., $Q^-_i$). B\&B then tries to ``fathom'' these leafs. At a high level, B\&B fathoms a leaf if it can guarantee that it will not find any better solution by branching on that leaf than the best solution found so far. See, for example, the research by \citet{Balcan18:Learning} for the formal protocol. Once B\&B has fathomed every leaf, it terminates. It returns the best feasible solution to $Q'$ that it found in the search tree, which is provably optimal.

\citet{Balcan18:Learning} focus on variable selection policies, and in particular, \emph{score-based variable selection policies}, defined below.
\begin{definition}[Score-based variable selection policy~\citep{Balcan18:Learning}]
Let $\score$ be a deterministic function that takes as input a partial search tree $\tree$, a leaf $Q$ of that tree, and an index $i$, and returns a real value $\score(\tree, Q, i) \in \R$. For a leaf $Q$ of a tree $\tree$, let $N_{\tree, Q}$ be the set of variables that have not yet been branched on along the path from the root of $\tree$ to $Q$. A score-based variable selection policy selects the variable $\argmax_{x_i \in N_{\tree, Q}} \{\score(\tree, Q, i)\}$ to branch on at the node $Q$.
\end{definition}
Score-based variable selection policies are extremely popular in B\&B implementations~\citep{Linderoth99:Computational,Achterberg09:SCIP,Gilpin11:Information}. See the research by \citet{Balcan18:Learning} for examples.
Given $d$ arbitrary scoring rules $\score_1, \dots, \score_d$, \citet{Balcan18:Learning} provide guidance for learning a linear combination $\rho_1\score_1 + \cdots + \rho_d\score_d$ that leads to small expected tree sizes. They assume that all aspects of the tree search algorithm except the variable selection policy, such as the node selection policy, are fixed. In their analysis, they prove the following lemma.

\begin{lemma}\label{lem:induction_general}[\citet{Balcan18:Learning}]
Let $\score_1, \dots, \score_d$ be $d$ arbitrary scoring rules and let $Q$ be an arbitrary MILP over $n$ binary variables. Suppose we limit B\&B to producing search trees of
  size $\tau$. There is a set $\mathcal{H}$ of at most $n^{2(\tau + 1)}$ hyperplanes such that for any connected component $R$ of $[0,1]^d \setminus \mathcal{H}$, the search tree B\&B builds using the scoring rule $\rho_1\score_1 + \cdots + \rho_d\score_d$ is invariant across all $(\rho_1, \dots, \rho_d) \in R$.
\end{lemma}

\citet{Balcan18:Learning} observe that in practice, the number of hyperplanes is significantly smaller.
Given a set $\sample$ of MILP instances, there are $|\sample|n^{2(\tau + 1)}$ relevant hyperplanes $\cH^*$. The number of connected components of the set $[0,1]^d \setminus \cH^*$ is at most $\left(|\sample|n^{2(\tau + 1)} + 1\right)^d$~\citep{Buck43:Partition}. Therefore, in our context, $\left|\gp(\sample, \tau)\right| \leq \left(|\sample|n^{2(\tau + 1)} + 1\right)^d$. When $d = 2$, \citet{Balcan18:Learning} provide guidance for finding the partition of $[0,1]$ into intervals $I$ where the search tree B\&B builds using the scoring rule $\rho \cdot \score_1 + (1 - \rho) \cdot \score_2$ is invariant across all $\rho \in I$. These intervals correspond to the output of the function $\gp$. An important direction for future research is extending the implementation to multi-dimensional parameter spaces.

\paragraph{Uniformly sampling configurations.}
\sampling*
\begin{proof}
We know that the probability $\left\{\rho_1, \dots, \rho_m\right\} \cap (a,b) = \emptyset$ is $(1 - (b-a))^m \geq (1 - (b-a))^{1/(b-a)} \geq \frac{1}{3}$ since $b-a \leq \frac{1}{6}$.
\end{proof}

\paragraph{Uniform convergence versus Algorithm~\ref{alg:findSubset}.}

We now describe an integer programming algorithm configuration problem where the best-known uniform convergence bound is exponential in the number of variables~\citep{Balcan18:Learning}, whereas our algorithm only requires $\tilde O\left(\left(\delta\eta\right)^{-2}\right)$ samples. We use a family of MILP distributions introduced by \citet{Balcan18:Learning}:

 \begin{theorem}[\citet{Balcan18:Learning}]\label{thm:WCdist}
For any MILP $Q$, let \[\score_1(\tree, \node, i) = \min\left\{\breve{c}_Q - \breve{c}_{Q_i^+}, \breve{c}_Q - \breve{c}_{Q_i^-}\right\}\] and \[\score_2(\tree, \node, i) = \max\left\{\breve{c}_Q - \breve{c}_{Q_i^+}, \breve{c}_Q - \breve{c}_{Q_i^-}\right\}.\] Define $\ell(\rho, Q)$ to be the size of the tree B\&B produces using the scoring rule $\rho \cdot \score_1 + (1 - \rho) \cdot \score_2$ given $Q$ as input. For every $a,b$ such that $\frac{1}{3} < a < b < \frac{1}{2}$ and for all even $n \geq 6$, there exists an infinite family of distributions $\dist$ over MILP instances with $n$ variables such that the following conditions hold:
\begin{enumerate}
\item If $\rho \leq a$, then $\ell(\rho, Q) = 2^{(n-5)/4}$ with probability $\frac{1}{2}$ and $\ell(\rho, Q)=8$ with probability $\frac{1}{2}$.
\item If $\rho \in (a,b)$, then $\ell(\rho, Q) = 8$ with probability $1$.
\item If $\rho \geq b$, then $\ell(\rho, Q) = 2^{(n-4)/2}$ with probability $\frac{1}{2}$ and $\ell(\rho, Q) = 8$ with probability $\frac{1}{2}$.
\end{enumerate} 
This holds no matter which node selection policy B\&B uses.
 \end{theorem}
 
As we describe in Section~\ref{sec:comparison}, the sample complexity bound \citet{Balcan18:Learning} provide implies that at least $\left(\frac{2^n}{\epsilon'}\right)^2$ samples are required to ensure that the average tree size branch-and-bound builds over the samples is within an additive $\epsilon'$ factor of the expected tree size. Even for $\epsilon' = 2^{n/2}$, this sample complexity bound is exponential in $n$.

\ourAlg*
\begin{proof}
In the proof of Theorem~\ref{thm:WCdist}, \citet{Balcan18:Learning} show that for any distribution $\dist$ in this family and any subset $\sample$ from the support of $\dist$, $\left|\gp(\sample, \tau)\right| = 3$, and the partition $\gp$ returns is $[0, a]$, $(a,b)$, and $[b, 1]$. Therefore, for the first three iterations, our algorithm will use $\tilde O\left(\left(\delta\eta\right)^{-2}\right)$ samples. At that point, $t = 3$, and the tree-size cap is 8. Algorithm~\ref{alg:findSubset} will discover that for all of the samples $Q \in \sample_3$, when $\rho \in (a,b)$, $\ell(\rho, Q) = 8$. Therefore, it add $(a,b)$ to $\cG$ and it will set $T = 8$. It will continue drawing $\tilde O\left(\left(\delta\eta\right)^{-2}\right)$ samples at each round until $2^{t-3}\delta \geq T = 8$, or in other words, until $t = O(\log(1/\delta))$. Therefore, the total number of samples it draws is $\tilde O\left(\left(\delta\eta\right)^{-2}\right)$. The set it returns will contain a point $\rho \in (a,b)$, which is optimal.
 \end{proof}

\subsection{Clustering}

We begin with an overview of agglomerative clustering algorithms.
A clustering instance $(V,d)$ consists of a set $V$ of $n$ points and a distance metric $d : V \times V \to \mathbb{R}_{\geq 0}$ specifying all pairwise distances between these points. The goal is to partition the points into groups such that distances within
each group are minimized and distances between each group are maximized.
Typically, the quality of a clustering is measured by an objective function, such as the classic $k$-means, $k$-median, or $k$-center objectives. Unfortunately, it is NP-hard to determine the clustering that minimizes any of these objectives.

An agglomerative clustering algorithm is characterized by a merge function $\merge(A,B) \to \R_{\geq 0}$, which defines the distance between any two sets of points $A,B \subseteq V$. 
The algorithm builds a \emph{cluster tree} $\mathcal{T}$, starting with $n$ singleton leaf nodes, each of which contains one point from $V$. The algorithm iteratively merges the two sets with minimum distance until
there is a single node remaining, consisting of the set $V$.
The children of any node $N$ in this tree correspond to the two sets of points that were merged to form $N$.  Common choices for the merge function $\merge$ include
$\min_{a\in A,b\in B} d(a,b)$ (single-linkage),  $\frac{1}{|A|\cdot |B|}\sum_{a\in A,b\in B}d(a,b)$ (average-linkage) and $\max_{a\in A,b\in B} d(a,b)$ (complete-linkage).
The linkage procedure is followed by a dynamic programming step, which returns the pruning of the tree that minimizes a fixed objective function, such as the $k$-means, $k$-median, or $k$-center objectives. If the linkage procedure is terminated early, we will be left with a forest, rather than a single tree. The dynamic programming procedure can easily be adjusted to find the best pruning of the forest (for example, by completing the hierarchy arbitrarily and enforcing that the dynamic programming algorithm only return clusters contained in the original forest).

\citet{Balcan17:Learning} define three infinite families of merge functions.
The families $\cA_1$ and $\cA_2$
consist of merge functions $\merge(A,B)$ that depend on
the minimum and maximum of all pairwise distances between $A$ and $B$.
The second family, denoted by $\mathcal{A}_3$, depends on all pairwise distances between $A$ and $B$.
All classes are parameterized by a single value $\rho$.
\begin{align*}
\cA_1&=\left\lbrace\left. \xi_{1, \rho} : (A,B) \mapsto \left(
\min_{u \in A, v \in B}(d(u,v))^{\rho} + \max_{u \in A, v \in B}(d(u,v))^
\rho \right)^{1/\rho}\, \right| \, \rho\in\mathbb{R}\cup\{\infty, -\infty\}\right\rbrace,\\
\cA_2&=\left\lbrace \left. \xi_{2, \rho} : (A,B) \mapsto \rho\min_{u\in A,v\in B}d(u,v)+(1-\rho)\max_{u\in A,v\in B}d(u,v)\, \right| \,
\rho\in[0,1]\right\rbrace,\\
\cA_3&=\left\lbrace \xi_{3, \rho} : (A,B) \mapsto \left(
\left. \frac{1}{|A||B|}\sum_{u \in A, v \in B} \left(d(u, v)\right)^{\rho}\right)^{1/\rho} \, \right| \, \rho \in \mathbb{R} \cup \{\infty, -\infty \}\right\rbrace.
\end{align*}

The classes
$\cA_1$ and $\cA_2$ define spectra of merge functions ranging
from single-linkage ($\xi_{1,-\infty}$ and $\xi_{2,1}$) to complete-linkage ($\xi_{1,\infty}$ and $\xi_{2,0}$). The class
$\cA_3$ includes average-, complete-, and single-linkage.

\begin{lemma}
Let $(V, d)$ be an arbitrary clustering instance over $n$ points. There is a partition of $\R$ into $k = O(n^8)$ intervals $I_1, \dots, I_k$ such that for any interval $I_i$ and any two parameters $\rho, \rho' \in I_i$, the sequences of merges the agglomerative clustering algorithm makes using the merge functions $\xi_{1, \rho}$ and $\xi_{1, \rho'}$ are identical. The same holds for the set of merge functions $\cA_2$.
\end{lemma}

Given a set $\sample$ of clustering instances, there are $O\left(|\sample|n^{8}\right)$ relevant intervals. Therefore, in our context, $\left|\gp(\sample, \tau)\right| = O(|\sample|n^8)$. \citet{Balcan17:Learning} provide guidance for finding these intervals, which correspond to the output of the function $\gp$. They also prove the following guarantee for the class $\cA_3$.

\begin{lemma}
Let $(V, d)$ be an arbitrary clustering instance over $n$ points. There is a partition of $\R$ into $k = O\left(n^23^{2n}\right)$ intervals $I_1, \dots, I_k$ such that for any interval $I_i$ and any two parameters $\rho, \rho' \in I_i$, the sequences of merges the agglomerative clustering algorithm makes using the merge functions $\xi_{3, \rho}$ and $\xi_{3, \rho'}$ are identical.
\end{lemma}

Similarly, given a set $\sample$ of clustering instances, there are $O\left(|\sample|n^23^{2n}\right)$ relevant intervals. Therefore, in our context, $\left|\gp(\sample, \tau)\right| = O\left(|\sample|n^23^{2n}\right)$. Again, \citet{Balcan17:Learning} provide guidance for finding these intervals.
\end{document}